\documentclass{article}

\PassOptionsToPackage{numbers, compress}{natbib} 

\usepackage[preprint]{neurips_2025}

\usepackage{wrapfig}
\usepackage{amsmath}
\usepackage{diagbox} 
\usepackage{slashbox}
\usepackage{multirow}  
\usepackage{makecell}  
\usepackage{microtype}
\usepackage{graphicx}
\usepackage{subfig}
\usepackage{bm}
\usepackage{booktabs} 
\usepackage{algorithm, algorithmic}
\newcommand{\fs}[1]{\scriptsize $\pm$#1}

\usepackage[utf8]{inputenc} 
\usepackage[T1]{fontenc}    
\usepackage{hyperref}       
\usepackage{url}            
\usepackage{booktabs}       
\usepackage{amsfonts}       
\usepackage{nicefrac}       
\usepackage{microtype}      
\usepackage{xcolor}         
\usepackage{graphicx}       

\usepackage{amssymb}
\usepackage{mathtools}
\usepackage{amsthm}
\usepackage[T1]{fontenc}
\usepackage[capitalize]{cleveref}

\theoremstyle{plain}
\newtheorem{theorem}{Theorem}[section]

\newtheorem{lemma}[theorem]{Lemma}
\newtheorem{corollary}[theorem]{Corollary}

\theoremstyle{definition}

\newtheorem{assumption}[theorem]{Assumption}
\theoremstyle{remark}

\crefname{lemma}{Lemma}{Lemmas}
\Crefname{lemma}{Lemma}{Lemmas}
\crefname{assumption}{Assumption}{Assumptions}
\Crefname{assumption}{Assumption}{Assumptions}

\title{Enhancing Gradient-based Discrete Sampling via Parallel Tempering}

\author{%
  Luxu LIANG\\
  School of Mathematics\\
  Renmin University of China\\
  \texttt{lianglux@ruc.edu.cn} \\
  \And
  Yuhang JIA\\
  Department of Mathematical Sciences\\
  Tsinghua University \\
  \texttt{jia-yh22@mails.tsinghua.edu.cn} \\
  \AND
  Feng ZHOU \\
  Center for Applied Statistics and School of Statistics \\
  Renmin University of China \\
  \texttt{feng.zhou@ruc.edu.cn} \\
}

\begin{document}

\maketitle

\begin{abstract}
While gradient-based discrete samplers are effective in sampling from complex distributions, they are susceptible to getting trapped in local minima, particularly in high-dimensional, multimodal discrete distributions, owing to the discontinuities inherent in these landscapes. To circumvent this issue, we combine parallel tempering, also known as replica exchange, with the discrete Langevin proposal and develop the Parallel Tempering enhanced Discrete Unadjusted Langevin Algorithm~(PT-DULA) and Parallel Tempering enhanced Discrete Metropolis Adjusted Langevin Algorithm~(PT-DMALA), which are simulated at a series of temperatures. Significant energy differences prompt sample swaps, which are governed by a Metropolis criterion specifically designed for discrete sampling to ensure detailed balance is maintained. Additionally, we introduce an automatic scheme to determine the optimal temperature schedule and the number of chains, ensuring adaptability across diverse tasks with minimal tuning. Theoretically, we establish both asymptotic and non-asymptotic convergence analyses of our algorithms. Empirical results further emphasize the superiority of our method in sampling from complex, multimodal discrete distributions, including synthetic problems, restricted Boltzmann machines, and deep energy-based models.
\end{abstract}

\section{Introduction}\label{introduction}
Discrete structures are prevalent in fields such as statistics~\citep{robert1999monte, doucet2000sequential, ait2003effects, neal2000markov, ishwaran2001gibbs}, physics~\citep{baumgartner2012monte, zarfaty2022discrete, negri2015efficient}, bioinformatics~\citep{bollback2006simmap, yu2013shrinkage, wang2010gibbs}, and computer science~\citep{wang2019bert, peters2018probabilistic, meng2022concrete, dawid2024introduction}, underscoring the need for efficient discrete samplers. Since direct sampling from a target probability distribution $\pi(\theta) \propto \exp(U(\theta))$ defined on a discrete space $\Theta$ is often intractable, Markov chain Monte Carlo (MCMC) methods are commonly employed. Recent advances~\citep{zanella2020informed, grathwohl2021oops, zhang2022langevin, sun2021path, sun2022optimal, sun2023any, xiang2023efficient, pynadath2024gradient} utilize gradient information within discrete distributions to enhance proposal distributions, significantly improving sampling efficiency.

A key limitation of gradient-based methods is their tendency to being trapped in local modes due to the reliance on gradient information~\citep{pynadath2024gradient,ziyin2021sgd}, particularly when dealing with well-separated modes, which hinders both accuracy and efficiency in sampling. In continuous domains, various techniques, such as parallel tempering~(PT)~\citep{chen2020accelerating,swendsen1986replica}, cyclical step sizes~\citep{zhang2019cyclical}, and flat histograms~\citep{berg1991multicanonical, deng2020contour}, have been proposed to mitigate this issue. Among these methods, PT is favored for its simplicity and parallelization. By simulating Langevin chains at varying temperatures and incorporating a swap mechanism, PT accelerates convergence while balancing exploration and exploitation. 

Despite their success in continuous domains, adapting such techniques to discrete spaces poses considerable challenges. Sampling from discrete multimodal distributions is even more challenging, as the discontinuous nature of the space inherently leads to more severe multimodality. Despite the urgent need, the development of effective gradient-based samplers capable of navigating such complex landscapes in discrete settings remains largely unexplored.

In this paper, we propose a method integrating PT with discrete Langevin sampling, enhancing the efficiency and accuracy of gradient-based samplers for discrete, multimodal distributions. Intuitively, the high-temperature chains serve as bridges, connecting different modes. To ensure detailed balance, we employ a tailored Metropolis step to determine swaps. To further improve practicality, we develop an automatic scheme for selecting temperature levels and the number of chains, making our method adaptable across various applications. Our contributions are summarized as follows:

1) We enhance the discrete Langevin proposal~\citep{zhang2022langevin} for multimodal distributions by incorporating PT, with optimized temperature schedules and chain configurations. The resulting method enables flexible, dataset-adaptive adjustments with minimal manual tuning, effectively balancing exploration and exploitation in discrete spaces.

2) We provide both asymptotic (including mixing time analysis, which may be of independent interest) and non-asymptotic convergence analyses of our algorithms, and theoretically establish a provably tighter lower bound on the convergence rate compared to DLP.

3) We demonstrate the superiority of our method for both sampling and learning tasks, including synthetic mixture models, restricted Boltzmann machines, and deep energy-based models.

\section{Related Works}\label{sec:2}
\paragraph{Gradient-based Discrete Sampling.}
Gradient-based discrete sampling has gained popularity for tackling complex discrete sampling tasks, with its origins rooted in Locally-Balanced Proposals~(LBP)~\citep{zanella2020informed}, which leverage local density ratios to enhance sampling efficiency. It has been extended to continuous-time Markov processes~\citep{power2019accelerated} and been used in Multiple-try Metropolis (MTM) algorithms~\citep{gagnon2023improving} to achieve fast convergence. \citet{grathwohl2021oops} expanded LBP by incorporating first-order Taylor approximations, ensuring computational feasibility and improving performance. To facilitate sampling in high-dimensional discrete spaces, LBP were further extended to explore larger neighborhoods through a sequence of small moves~\citep{sun2021path}. \citet{zhang2022langevin} proposed Discrete Langevin Proposal~(DLP) by adapting continuous Langevin MCMC methods to discrete spaces, allowing parallel updates of all coordinates based on gradient information.~\citet{sun2023discrete} further generalize Langevin Monte Carlo (LMC) to discrete spaces via Wasserstein gradient flow, deriving the Discrete Langevin Monte Carlo (DLMC) algorithm, which further improves sampling efficiency. Additionally, DLP has also been refined with an adaptive mechanism to automatically adjust step sizes for better efficiency~\citep{sun2023any}. While these approaches have achieved notable success, sampling from discrete, complex, multimodal distributions remains a significant challenge.


\paragraph{Sampling on Multimodal Distributions.}
Various algorithms have been proposed to enhance exploration in complex, multimodal distributions, including importance sampling~\citep{wang2001efficient}, simulated annealing~\citep{kirkpatrick1983optimization}, simulated tempering~\citep{marinari1992simulated}, cyclic step-size scheduling~\citep{zhang2019cyclical}, dynamic weighting~\citep{wong1997dynamic}, and replica exchange Monte Carlo~\citep{earl2005parallel,swendsen1986replica}. Among these, simulated annealing and simulated tempering SGMCMC~\citep{ge2018simulated} accelerate convergence with dynamic temperatures. However, simulated annealing is sensitive to fast-decaying temperatures, and simulated tempering requires approximating the normalizing constant. Replica exchange MCMC (reMCMC) uses multiple chains at different temperatures with exchanges, offering easier implementation and parallelism. Studies have analyzed reMCMC’s acceleration effect~\citep{chen2020accelerating}, spectral gap properties~\citep{dong2022spectral}, and efficiency in deep learning~\citep{deng2020non, deng2020accelerating}. To the best of our knowledge, although PT has shown promise in continuous Langevin dynamics, and discrete domains often exhibit more severe multimodality due to inherent discontinuities, the potential of PT to improve gradient-based samplers in multimodal discrete domains remains untapped.~\citet{pynadath2024gradient} proposed a cyclic scheduling strategy that alternates step sizes, enhancing the handling of multimodal distributions. \citet{zhengexploring} attempted to integrate replica exchange with gradient-based sampling; however, their approach lacks a rigorous theoretical foundation, and the two replicas encounter a specific issue, as discussed in \cref{sec_4_1}.

\section{Preliminaries}
This section provides a formal definition of the problem and reviews relevant methods.
\subsection{Problem Definition}
We aim to sample from a discrete target distribution $\pi: \Theta \to [0, 1]$ defined as
\begin{equation*}
\pi(\theta) = \frac{1}{Z} \exp\left(U(\theta)\right), \quad \theta \in \Theta \subseteq \mathbb{R}^d,
\end{equation*}
where $U$ is the energy function, and $Z$ the normalizing constant. Following standard settings in gradient-based discrete sampling~\citep{grathwohl2021oops, zhang2022langevin}, the domain $\Theta$ is finite and coordinate-wise factorized, i.e., $\Theta = \prod_{i=1}^d \Theta_i$, with typical choices including binary $\{0,1\}^d$ and categorical $\{0,1,\dots,N\}^d$ spaces. The energy function $U$ is assumed differentiable\footnote{Noted that this assumption can be relaxed via Newton’s Series Approximation~\citep{xiang2023efficient}.} across $\mathbb{R}^d$.

\subsection{Replica Exchange Langevin Dynamics}
The replica exchange Langevin Dynamics~(reLD) is a widely used sampling method for non-convex exploration in continuous spaces. The method updates according to the following dynamics, for $k = 1, \dots, K$ and $i = 1, 2, \dots, n$,
\[
\theta^{(k)}_{i+1} = \theta^{(k)}_i + \frac{\alpha_k}{2} \nabla U(\theta^{(k)}_i) + \sqrt{\frac{\alpha_k}{\beta_k}} \xi_k,
\]
where $\{\alpha_k\}_{k=1}^{K}$ represent the step sizes, $\{\beta_k\}_{k=1}^{K}$ are the inverse temperature parameters, and $\{\xi_k\}_{k=1}^{K}$ are independent Gaussian noises drawn from $\mathcal{N}(0, I_{d \times d})$. In the typical set-up, the first chain is designated as the low-temperature chain. The gradient $\nabla U(\cdot)$ guides the algorithm toward high-probability regions. To further improve the mixing rate over Langevin dynamics, reLD enables interaction through a chain-swap mechanism between neighboring replicas. Specifically, the probability to swap the $i$-th samples between $\theta_i^{(k)}$ and $\theta_i^{(k+1)}$ is determined by $ s_k: \Theta \times \Theta \to \mathbb{R}^+$, which is given by, for $k=1, \cdots, K-1$,
\begin{equation}\label{traditional_swap}
\!s_k\!\left(\theta_i^{(k)}, \theta_i^{(k+1)}\!\right)\!=\! \min\!\left\{1,  e^{\left(\beta_{k}- \beta_{k+1}\right)\! \left[U(\theta_i^{(k+1)}) - U(\theta_i^{(k)})\right]}\!\right\}.
\end{equation}
Intuitively, the probability of swap in reLD depends on the energy values in $\theta_i^{(k)}$ and $\theta_i^{(k+1)}$. When the low-temperature chain is trapped in a local minimum and the high-temperature chain explores modes with much lower energy, swapping allows the former to escape and characterize new modes, while the latter continues broader exploration. 

\subsection{Discrete Langevin Sampler}
The Discrete Langevin Proposal (DLP)~\citep{zhang2022langevin} is a gradient-based method for sampling from high-dimensional discrete distributions. For a target distribution $\pi(\theta) \propto \exp\left(U(\theta)\right)$, DLP proposes a new sample $\theta'$ based on a Taylor expansion:
$$
q(\theta^{\prime} | \theta) = \frac{\exp\left(-\frac{1}{2\alpha} \left\| \theta^{\prime} - \theta - \frac{\alpha}{2} \nabla U(\theta) \right\|_2^2 \right)}{Z_\Theta(\theta)}, 
$$
where $\theta, \ \theta^{\prime}\in \Theta$, $\nabla U(\theta)$ is the gradient of the energy function evaluated at $\theta$, and $Z_\Theta(\theta)$ normalizes the distribution. A key insight is that, for $i=1, \cdots, d$, the update rule can be factorized by coordinate:
\begin{equation}\label{factorize}
\text{Cat} \left[ \text{Softmax} \left( 
\frac{1}{2} \nabla U(\theta)_i (\theta_{i}^{\prime} - \theta_{i}) 
- \frac{1}{2\alpha}(\theta_{i}^{\prime} - \theta_{i})^2 
\right)\right], 
\end{equation}
with $\theta_i^{\prime} \in \Theta_i$, DLP remains scalable and efficient for complex distributions. It can be used with or without the Metropolis-Hastings (M-H) step, corresponding to DMALA and DULA~\citep{zhang2022langevin}, respectively.

\section{Methodology}\label{sec_4}
In this section, we introduce our proposed algorithms in~\cref{sec_4_1,sec_4_2}, and discuss the optimal temperature schedule and the number of chains in~\cref{sec_4_3}.

\subsection{Parallel Tempering Enhanced Discrete Langevin Proposal}\label{sec_4_1}


\begin{wrapfigure}{r}{0.6\textwidth}
  \centering
  \vspace{-15pt}
  \subfloat{
    \includegraphics[width=0.29\textwidth]{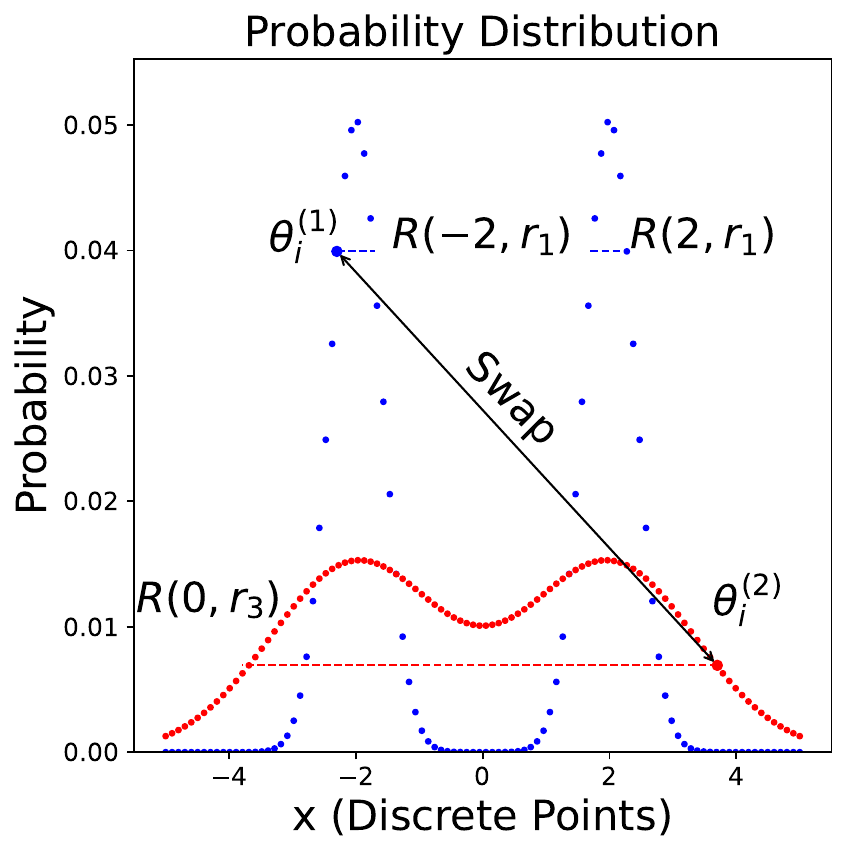}
  }
  \subfloat{
    \includegraphics[width=0.29\textwidth]{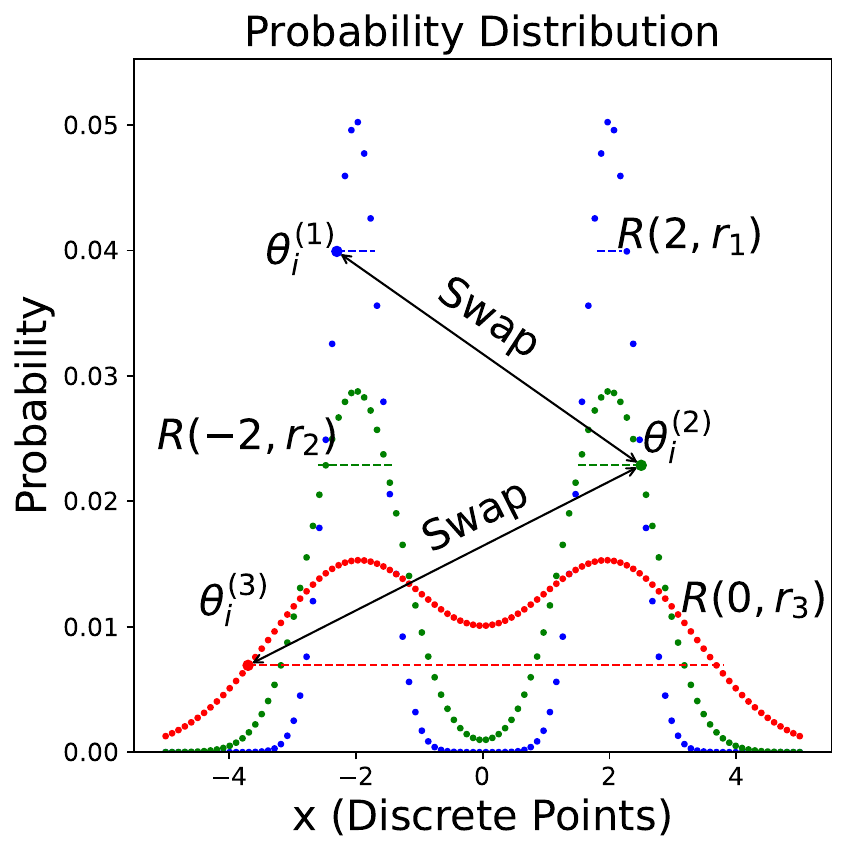}
  }
  \caption{The blue, green, and red dots correspond to probability functions at three temperatures. The high-probability areas to sample from are indicated by dashed lines.}
  \label{fig:intuition}
  \vspace{-15pt}
\end{wrapfigure}

One major issue with two replicas is that the swaps may not happen often enough. To see this, denote by $R(r, M):=\{x:\|x - r\| \leq M\}$. As shown in \cref{fig:intuition},
the figure on the left illustrates the swap between two chains. The high-probability regions for $\theta_i^{(1)}$ and $\theta_i^{(2)}$ are defined by $R\left(2, r_1\right) \cup R\left(-2, r_1\right)$ and $R\left(0, r_3\right)$. Swaps between $\theta_i^{(1)}$ and $\theta_i^{(2)}$ are unlikely to occur frequently, as $\theta_i^{(2)}$ has a low probability of falling within the region $R\left(2, r_1\right) \cup R\left(-2, r_1\right)$. However, when the number of chains increases to three, the high-probability region for $\theta_i^{(2)}$ becomes $R\left(2, r_2\right) \cup R\left(-2, r_2\right)$ with $r_1 < r_2 < r_3$, making it easier for $\theta_i^{(3)}$ to lie within this region, thereby increasing the frequency of swaps. In light of the fact that non-adjacent chain swaps are unlikely to occur, we exclusively consider adjacent swaps in this paper.

Building on the previous discussion, we propose a method that incorporates multiple chains to further enhance performance:
\begin{equation}\label{PTDLP_1}
\textbf{Exploitation:}\ q_1\!\left(\theta^{\prime} \!\mid \!{\theta}\right)\! \propto\! \exp \bigg\{\frac{\beta_1}{2} \nabla U({\theta})^{\top}\left({\theta}^{\prime}\!-\!{\theta}\right)-\frac{1}{2 \alpha_1}\left\|{\theta}^{\prime}-{\theta}\right\|_p^p\bigg\},
\end{equation}
\begin{equation}\label{PTDLP_2}
\textbf{Exploration:} \ q_k\!\left({\theta}^{\prime}\! \mid\! {\theta}\right)\!\propto\! \exp \bigg\{\underbrace{\frac{\beta_k}{2} \nabla U({\theta})^{\top}\left({\theta}^{\prime}\!-\!{\theta}\right)}_{\text {First-order Taylor Expansion}}-\underbrace{\frac{1}{2 \alpha_k}\left\|{\theta}^{\prime}-{\theta}\right\|_p^p}_{\text {Regularizer}}\bigg\},
\end{equation}
where $k=2,\cdots, K$ and $1=\beta_1>\cdots >\beta_K\geq0$. Note that the above proposals can also be factorized by coordinates, as shown in \cref{factorize}, which allows us to update each coordinate in parallel after computing $\nabla U(\theta)$. \citet{zhang2022langevin} emphasize the importance of the regularizer term, as it introduces a parameter similar to the step size. Note that we have chosen the 
$\textit{p}$-norm instead of the $2$-norm, considering that in certain tasks, selecting alternative norms may improve model performance, which could be due to the geometric structure of specific discrete domain distributions~\citep{jiang2024uncovering,park2023linear}. The exchange takes place between neighboring replicas.  In particular, for each $1 \leq k \leq K-1$, $\theta_{i+1}^{(k)}$ and $\theta_{i+1}^{(k+1)}$ are swapped according to a \textit{tailored Metropolis criterion} $ s_k$, which is given by 
\begin{equation}\label{swap_function}
s_k\left(\theta_{i+1}^{(k)}, \theta_{i+1}^{(k+1)}\mid \theta_i^{(k)}, \theta_i^{(k+1)}\right)=\! \min\left\{\!1, e^{\beta_{\delta, k}\left[U\left({\theta}_{i+1}^{(k+1)}\right)+U\left({\theta}_i^{(k+1)}\right)-U\left({\theta}_{i+1}^{(k)}\right)-U\left({\theta}_i^{(k)}\right)\right]}\!\right\}, 
\end{equation}
where $\beta_{\delta, k}:= \beta_{k} - \beta_{k+1}$. The traditional swap rate defined in~\Cref{traditional_swap} used in reLD relies on a \textit{decaying} step size to ensure that the stationary distribution approximates the target distribution. However, such a technique is not applicable in the discrete domain. Asymptotic convergence to the target distribution with fixed step sizes requires that detailed balance be preserved not only between the low- and high-temperature samplers, but also between successive output samples. The validation of the tailored criterion will be established in~\cref{sec_5}. We denote the proposal in \Cref{PTDLP_1,PTDLP_2,swap_function} by Parallel Tempering enhanced Discrete Unadjusted Langevin Algorithm~(PT-DULA). This approach involves running multiple chains in parallel, with each chain exploring a unique region of the parameter space. By exchanging information through swaps, the chains can effectively traverse diverse areas of the solution space, reducing the risk of becoming trapped in local minima. 

\paragraph{Local M-H Correction.}
It is optional to add M-H corrections~\citep{metropolis1953equation} for local kernels, which is usually combined with proposals to make the Markov chain reversible. Specifically, for each $k=1, \cdots, K$, after generating the next position $\theta^{\prime}$ from $q_k(\cdot \mid \theta)$, the M-H step accepts it with the probability: 
\begin{equation}\label{MH}
\min \left\{1, \exp \left(\beta_k\left(U\left(\theta^{\prime}\right) - U(\theta)\right)\right) \frac{q_k\left(\theta \mid \theta^{\prime}\right)}{q_k\left(\theta^{\prime} \mid \theta\right)}\right\}.
\end{equation}
This ensures that the marginal distribution of each replica $\theta^{(k)}$ admits the invariant distribution $\pi^{\beta_k}(\theta)\propto \exp(\beta_k U(\theta))$. We refer to the resulting method, which incorporates local M-H corrections within the parallel tempering framework, as the Parallel Tempering-enhanced Discrete Metropolis-Adjusted Langevin Algorithm (PT-DMALA). Each local kernel in PT-DMALA requires two gradient and two function evaluations, whereas PT-DULA involves only a single gradient evaluation at the cost of potential asymptotic bias, making it more suitable when the M-H step is costly. A stochastic gradient variant designed for large-scale datasets will be introduced in the following subsection.

\subsection{PT-DULA in Mini-Batch Setting}\label{sec_4_2}
As mentioned earlier, the methods discussed above require the evaluation of the energy function and gradient based on the full dataset, which is not scalable to large data~\citep{deng2020accelerating, lin2022multi}. Similar to Stochastic Gradient Langevin Dynamics (SGLD)~\citep{welling2011bayesian}, we replace the full-batch energy function and gradient with the unbiased stochastic estimators $\tilde{U}(\cdot)$ and $\nabla \tilde{U}$ in PT-DULA, thereby reducing the computational cost of our method for large-scale problems. Directly replacing the energy function and gradient of PT-DULA with their stochastic counterparts introduces significant bias. Intuitively, assuming that \(\tilde{U}(\cdot) \sim N(U(\cdot), \sigma^2)\) and denoting the stochastic version of \(s_k\) by \(\tilde{s}_k\), we can apply Jensen's inequality to obtain $\mathbb{E}[e^{a\tilde{U}(\cdot)}] \geq e^{a\mathbb{E}[\tilde{U}(\cdot)]}$ for $a>0$, with strict inequality holding when \(\tilde{U}(\cdot)\) is a random variable. Motivated by~\citet{deng2020accelerating}, we propose the following swapping rate:
\begin{equation}\label{stochastic_swap}
\tilde{s}_k\left(\theta_{i+1}^{(k)}, \theta_{i+1}^{(k+1)}\mid \theta_i^{(k)}, \theta_i^{(k+1)}\right)= \min\left\{1, e^{\beta_{\delta, k}\left[\tilde{U}\left({\theta}_{i+1}^{(k+1)}\right)+\tilde{U}\left({\theta}_i^{(k+1)}\right)-\tilde{U}\left({\theta}_{i+1}^{(k)}\right)-\tilde{U}\left({\theta}_i^{(k)}\right) - \beta_{\delta, k} \sigma^2\right]}\right\},
\end{equation}
where the factor $\beta_{\delta, k} \sigma^2$ in the exponent is used to correct the bias caused by the incorrect estimation of the energy function\footnote{Note that this swapping rate is not exactly unbiased, since $\mathbb{E}[\min\{1, \hat{s}_k\}] \leq \min\{1, s_k\}$. It was found in~\citet{deng2020accelerating} that this correction works well for most problems.}. As the number of chains increases, additional parameters, such as the number of chains and the temperature schedule, must be specified, as discussed in the following section.

\subsection{Warm-up Phase}\label{sec_4_3}
\paragraph{Optimal Temperature Schedule.}
Poor temperature spacing can cause replica systems to be too distant, hindering exchanges, or too close, limiting diversity~\citep{kone2005selection}. To address this, we aim to optimize the temperature schedule by maximizing the round-trip rate—the expected frequency with which a replica travels from the lowest to the highest temperature and back. Following~\citet[Assumption 2]{syed2022non}, we have the round-trip rate of our algorithm in~\cref{round_trip_rate}, which shows that maximizing the round-trip rate is equivalent to minimizing $\sum_{k=1}^{K-1} \frac{1}{s_k}$. Moreover, $\sum_{k=1}^{K-1} (1 - s_k)$ converges to a fixed barrier $\Lambda$ as $K \to \infty$~\citep{predescu2004incomplete, syed2022non}. Lagrange multipliers yield equal transition probabilities: $s_1 = s_2 = \cdots = s_{K-1}$. We estimate the barrier function $\Lambda(\beta)$ from a pilot run, interpolate it to obtain $\hat{\Lambda}(\beta_k)$. The optimal schedule is then determined through~\citet[Eq.(30)]{syed2022non} and a bisection method, resulting in the temperature set $\mathcal{T}_K^{\ast} = \left\{\beta_1^{\ast}, \ldots, \beta_K^{\ast}\right\}$.

\paragraph{Optimal Chain Number.}
Given a fixed temperature schedule, we now optimize the number of chains. Suppose $\mathcal{B}$ parallel PT instances are run, each with $K$ chains using the optimal schedule. Let $K_{\textit{total}}$ be the total number of available computational units, subject to the constraint $\mathcal{B}K \leq K_{\textit{total}}$.~\citet{nadler2007dynamics,syed2022non} demonstrated that the non-asymptotic (with respect to $K$) round-trip rate of the reversible PT scheme is
$$
\tau_{\mathcal{B}}(K)= \mathcal{B} / \sum_{k=1}^{K-1} \frac{1}{s_k}.
$$
The next lemma explains how to determine the optimal number of chains, given the optimal temperature schedule.
\begin{lemma}\label{lemma_4_2}
$\tau_{\mathcal{B}}(\cdot)$ is optimized when we run $\mathcal{B}^*= \lfloor K_{\text{total}} / K^* \rfloor$ copies of PT with $K^{\ast}=2 \Lambda + 1$. 
\end{lemma}
Detailed proofs and the schedule tuning algorithm are given in \cref{tuning_algo,app_D_1}.

\section{Theoretical Analysis}\label{sec_5}
In this section, we present an asymptotic convergence and mixing time analysis of PT-DULA, along with a non-asymptotic convergence analysis of PT-DMALA. These results extend prior analyses~\citep{grathwohl2021oops,pynadath2024gradient,zhang2022langevin}, and further demonstrate the acceleration gains enabled by the swap mechanism.

\subsection{Asymptotic Convergence Analysis}\label{sec_5_1}
First, we prove the asymptotic convergence of PT-DULA. \citet{zhang2022langevin} showed that a discrete Langevin-like sampler with temperature 1 is reversible for log-quadratic energy distributions with small step sizes. However, this does not directly extend to our proposed algorithm due to the multi-chain structure, swap mechanism, and higher temperatures. In this section, we extend the proof to PT-DULA and focus on the case of three chains, with the result extendable to more.



\begin{theorem}\label{thm_5_1}
    Let $\pi(\theta) \propto \exp( U(\theta))$ be the target distribution and $\tilde{\pi} (\theta)\propto \exp \left(\theta^{\top} W \theta+b^{\top} \theta\right)$ be the log-quadratic distribution satisfying that $\exists\ W \in \mathbb{R}^{d \times d}$, $b \in \mathbb{R}$, $\epsilon \in \mathbb{R}^{+}$, such that $\|\nabla U(\theta)-(2 W \theta+b)\|_1 \leq \epsilon$ for any $\theta \in \Theta$. Then the stationary distribution $\pi_\alpha$ of PT-DULA satisfies
\begin{equation}\label{thm_5_1_eq}
\left\|\pi_\alpha-\pi\right\|_{TV} \leq Z_1\exp\left(Z_2\epsilon\right) + Z_3 \exp \left(-\frac{1+\alpha \lambda_{\min }}{2 \alpha}\right) - Z_1, 
\end{equation}
where $\left\|\cdot\right\|_{TV}$ is the total variation distance, $\lambda_{\min}$ the smallest eigenvalue of $W$, $Z_1$ a constant depending on $\tilde{\pi}$ and $\alpha$, $Z_2$ on $\Theta$ and $\max\limits_{\theta, \theta' \in \Theta} \left\|\theta' - \theta\right\|_{\infty}$, and $Z_3$ a constant associated with $\tilde{\pi}$.
\end{theorem}
\Cref{thm_5_1} demonstrates that the tailored swap function defined in~\Cref{swap_function} guarantees the asymptotic convergence of PT-DULA. The low bias of PT-DULA, as quantified by~\cref{thm_5_1}, implies that $\pi_{\alpha}$ closely approximates $\pi$, leading to higher acceptance rates in the local M-H step of PT-DMALA and improved efficiency. The next theorem establishes upper and lower bounds on the algorithm's mixing time. Denote by 
$$
d_p := \inf\limits_{\theta \neq \theta^{\prime} \in \Theta}\left\|\theta-\theta^{\prime}\right\|_p^p, \ \mathcal{D}_p := \sup\limits_{\theta, \theta^{\prime} \in \Theta}\left\|\theta-\theta^{\prime}\right\|_p^p.
$$
\begin{theorem}\label{thm:5.2}
If the target distribution is assumed to be log-quadratic, i.e., for any $\theta \in \Theta$, $\pi(\theta) \propto \exp \left(\theta^{\top} W \theta+b^{\top} \theta\right)$ with some constants $W \in \mathbb{R}^{d \times d}$ and $b \in \mathbb{R}^{d}$. The mixing time of PT-DULA satisfies
$$
\mathcal{L}\leq t_{\operatorname{mix}}(\varepsilon)\leq \mathcal{U},
$$
where
$$\mathcal{L}\!:=\!\left(\!\frac{1}{4 Z}\exp \!\left(\frac{1}{2} \lambda_{\min }(W) d_2+\frac{1}{2 \alpha} d_p\right)\!-\!1\!\right) \log \left(\!\frac{1}{2 \varepsilon}\!\right),\, \mathcal{U}\!:=\!\frac{2}{\left(I_{\pi_{\alpha}}(\Theta) \pi_{\alpha,\min} q_{\min}\right)^2}\log(\frac{1}{\epsilon\pi_{\alpha,\min}}),$$
$\lambda_{\min}$ is the smallest eigenvalue of $W$, $I_{\pi_{\alpha}}(\Theta)$ denotes the Cheeger constant associated with $\Theta$ and $\pi_\alpha$, $\pi_{\alpha,\min} := \min\limits_{\theta \in \Theta} \pi_\alpha(\theta)$, and $q_{\min}$ is defined in~\Cref{thm_5_2_eq3}.
\end{theorem}


Without the M-H step, the discrepancy between the algorithm's stationary distribution and the target distribution stems from two sources: the approximation to a log-quadratic distribution and the step size error. As noted by~\citet{zhang2022langevin}, this error becomes negligible for sufficiently small step sizes, and omitting the M-H step reduces computational cost. However, \cref{thm:5.2} shows that as the step size decreases, the lower bound on the mixing time grows exponentially, offsetting the computational gain. Moreover, decaying step sizes common in continuous domains are not feasible in discrete settings. Thus, to ensure convergence with fixed step sizes, incorporating the M-H step is favored.


\subsection{Non-asymptotic Convergence Analysis}\label{sec_5_2}
Next, we focus on proving the non-asymptotic convergence of PT-DMALA. Our primary proof strategy is to establish the uniform ergodicity of PT-DMALA by constructing a uniform minorization condition. For simplicity of the proof, we consider the case of three chains. The corresponding transition kernel, denoted by $p(\cdot \mid \cdot)$, is given in~\Cref{pt_dmala_p}. 
\begin{theorem}\label{thm:4}
Let~\Cref{asm_2,asm_3} hold. Let \(P\) denote the Markov transition operator with kernel \(p(\theta' \mid \theta^{(1)}_i)\). Then, for the Markov chain \(P\) with three chains, and for any \(\theta', \theta^{(1)} \in \Theta\), we have
\[
p(\theta' \mid \theta^{(1)}_i) \geq \epsilon \frac{\exp\left\{ \beta_3 U(\theta') \right\}}{\sum_{\theta' \in \Theta} \exp\left\{ \beta_3 U(\theta') \right\}},
\]
where \(\epsilon := \epsilon_0^2 \epsilon_{\beta_3, \alpha}\), with \(\epsilon_0\) and $\epsilon_{\beta_3, \alpha}$ defined in~\cref{epsilon_0,epsilon_3}. Then \(P\) is uniformly ergodic, i.e.,
\[
\|P^n - \pi\|_{\mathrm{TV}} \leq (1 - \epsilon)^n.
\]
\end{theorem}


\cref{thm:4} proves the non-asymptotic convergence of PT-DMALA. In the next corollary, we will examine the impact of the swap mechanism.
\begin{corollary}\label{col:5.7}
Let \cref{asm_2,asm_3} hold. Assume that $$\|\nabla U(a)\| < \left((M-\frac{m}{2})\mathcal{D}_2 - 2\log(1 / \epsilon_0)\right)/ \mathcal{D}_1,$$
where $0 < \epsilon_0 < 1$ is from~\Cref{epsilon_0} and $a:=\arg\min_{\theta\in\Theta}\|\nabla U(\theta)\|$. Then, PT-DMALA provides a better guaranteed upper bound on convergence speed compared to DLP.
\end{corollary}

\section{Experiments} 
In this section, we evaluate the newly proposed method on four problem types: \textbf{(1)} sampling from synthetic distributions, \textbf{(2)} sampling from restricted Boltzmann machines on real-world datasets, \textbf{(3)} learning restricted Boltzmann machines and \textbf{(4)} deep energy-based models parameterized by a convolutional neural network. 

For sampling tasks, we compare our algorithm with several popular gradient-based discrete samplers: the discrete Langevin-like samplers~(\textbf{DULA} and \textbf{DMALA}) \citep{zhang2022langevin}, the any-scale balanced sampler (\textbf{AB}) \citep{sun2023any}, and the automatic cyclical sampler (\textbf{ACS}) \citep{pynadath2024gradient}. 
For learning tasks, we exclude the AB sampler, as it is not originally designed for model learning applications. More details such as experimental setups, hyperparameters, and additional experimental results can refer to~\cref{App_ex}. We released the code of the synthetic task at \href{https://anonymous.4open.science/r/PTDLP-73AD}{https://anonymous.4open.science/r/PTDLP-73AD}. 

\subsection{Synthetic Problems}
\begin{wraptable}{r}{0.6\textwidth}
    \vspace{-10pt}
    \caption{Results of exploring MoG and MoS, measured by KL and MMD (c denotes the number of components).}
    \label{kl_mmd}
    \centering
    \resizebox{1\linewidth}{!}{
    \begin{tabular}{cccccc}
    \toprule
        \multicolumn{2}{c}{Metrics / Samplers}  & MoG~(c=8) & MoG~(c=16) & MoS~(c=8) & MoS~(c=16)\\
    \midrule
        \multirow{4}{*}{\makecell{MMD~($10^{-3}$)($\downarrow$)}} 
        & DMALA & 1.214 \fs{0.058}&2.130\fs{0.064} &1.617\fs{0.061}& 2.158 \fs{0.073}  \\
        & ACS &0.984 \fs{0.031} &1.806 \fs{0.056} &1.406 \fs{0.057}&1.813 \fs{0.061}\\
        & AB &0.891 \fs{0.026}  & 1.691 \fs{0.042} & 1.305 \fs{0.044} &1.515 \fs{0.068}\\
        & PT-DMALA~\textit{(Ours)} &$\textbf{0.534 \fs{0.015}}$  & \textbf{0.824 \fs{0.031}}&  \textbf{0.744\fs{0.028}} & \textbf{0.941 \fs{0.022}} \\
    \midrule
        \multirow{4}{*}{\makecell{KL~($10^{-2}$) ($\downarrow$)}}
        & DMALA &1.331 \fs{0.032}&7.660 \fs{0.043} & 2.017 \fs{0.026}&7.674 \fs{0.029}\\
        & ACS & 0.662\fs{0.012} & 2.177 \fs{0.023}& 3.117 \fs{0.041}&3.112 \fs{0.017}\\
        & AB & 0.851 \fs{0.011} & 3.216 \fs{0.022}& 2.801 \fs{0.026}&2.871 \fs{0.021}\\
        & PT-DMALA~\textit{(Ours)} &\textbf{0.617 \fs{ 0.009}} & \textbf{2.133 \fs {0.017}} & \textbf{0.667 \fs {0.017}} & \textbf{1.967 \fs {0.014}}\\
    \bottomrule
    \end{tabular}}
\end{wraptable}
We first address the challenges of sampling from two-dimensional discrete multimodal distributions, specifically mixture of Gaussian components~(MoG) and mixture of Student's t-distributions~(MoS). The two-dimensional continuous domain was discretized by partitioning each axis into 100 intervals, followed by sampling over the resulting discrete space. To quantify the ability to avoid mode collapse, we use Entropic Mode Coverage (EMC)~\citep{blessing2024beyond}, Maximum Mean Discrepancy (MMD)~\citep{gretton2012kernel}, and forward Kullback-Leibler divergence~\citep{kullback1951information} as quantitative performance metrics. Notably, EMC $\in[0,1]$ serves as a heuristic metric for mode collapse detection, where a value of $0$ indicates samples come from a single mode, while a value of 1 suggests that all modes are adequately covered.

\paragraph{Results and Analysis.}
As shown in \Cref{kl_mmd} and the left two plots of \Cref{fig_synthetic}, our algorithm consistently outperforms existing methods on both MoG and MoS across varying component counts and evaluation metrics. While methods such as AB and ACS mitigate mode-trapping using variable step sizes, their improvements are limited. The right two plots of \Cref{fig_synthetic} further illustrate that, given the same number of iterations, our method enables more effective exploration and produces higher-quality samples. The performance gain stems from parallel chains at varying temperatures, enabling traversal of isolated modes and reducing local trapping.
\begin{figure}[ht]
  \begin{center}
    \begin{minipage}{0.23\textwidth}
      \includegraphics[width=\textwidth]{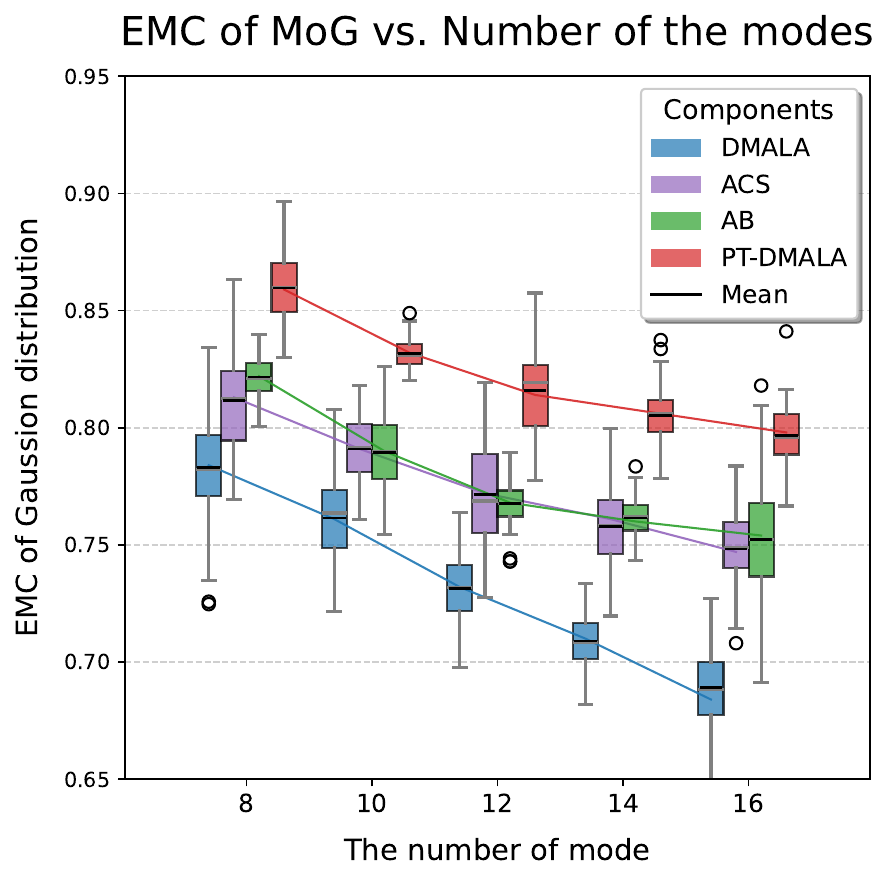}
    \end{minipage}
    \begin{minipage}{0.23\textwidth}
      \includegraphics[width=\textwidth]{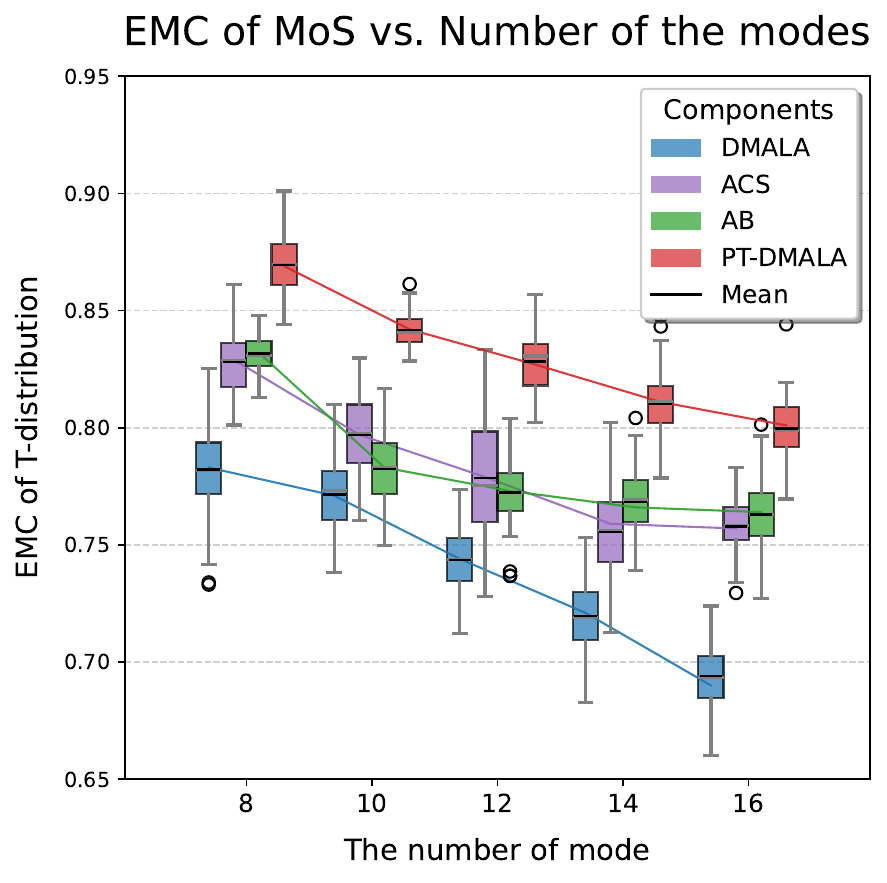}
    \end{minipage}
        \begin{minipage}{0.23\textwidth}
      \includegraphics[width=\textwidth]{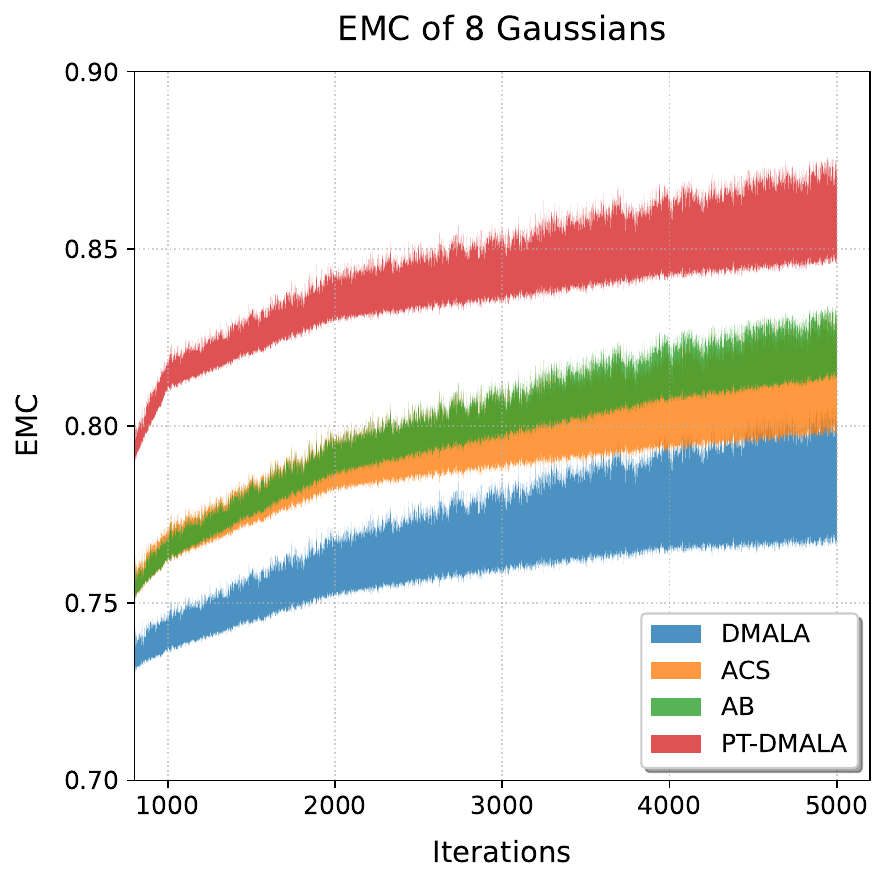}
    \end{minipage}
        \begin{minipage}{0.23\textwidth}
      \includegraphics[width=\textwidth]{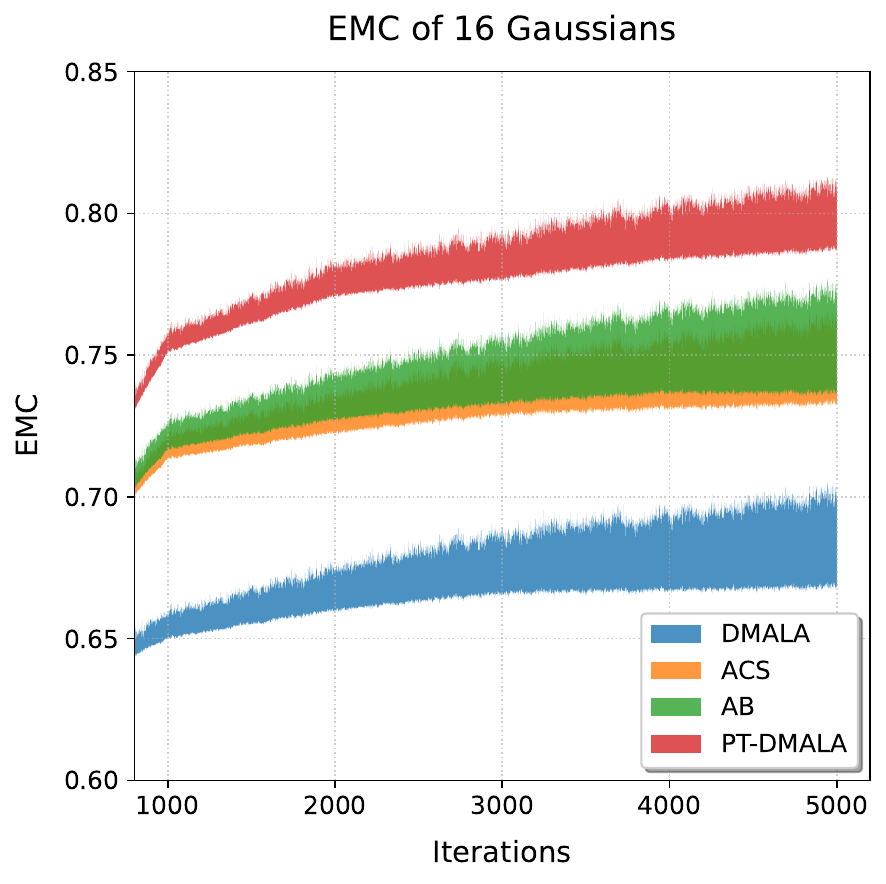}
    \end{minipage}
  \end{center}
  \caption{Sampling performance (measured by EMC) of various methods for MoG (left) and MoS (right) with varying components. Sampling performance of various interations for 8 Gaussions and 16 Gaussions. PT-DMALA consistently outperforms baselines across random seeds.}
  \label{fig_synthetic}
\end{figure}

\begin{figure}[ht]
  \begin{center}
    \begin{minipage}{0.195\textwidth}
      \includegraphics[width=\textwidth]{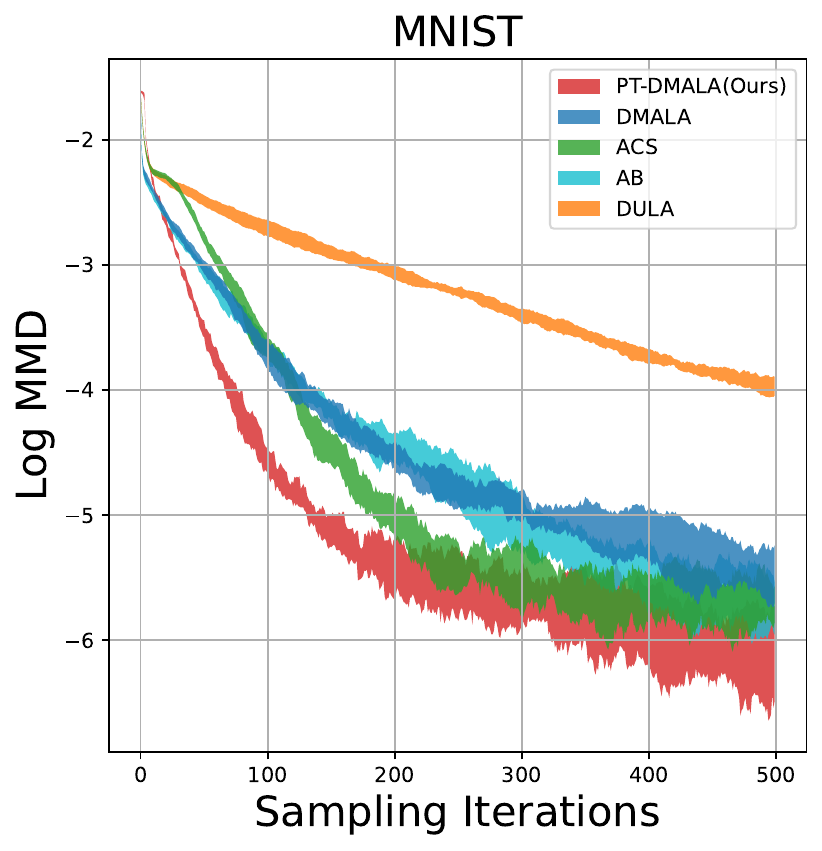}
    \end{minipage}
    \begin{minipage}{0.195\textwidth}
      \includegraphics[width=\textwidth]{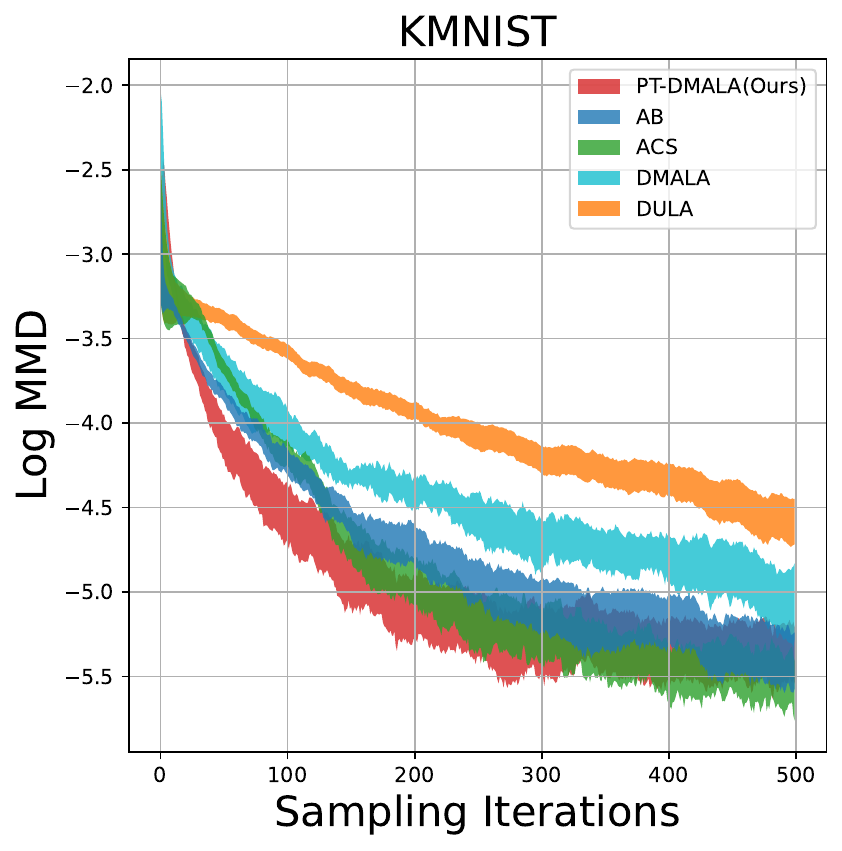}
    \end{minipage}
    \begin{minipage}{0.195\textwidth}
      \includegraphics[width=\textwidth]{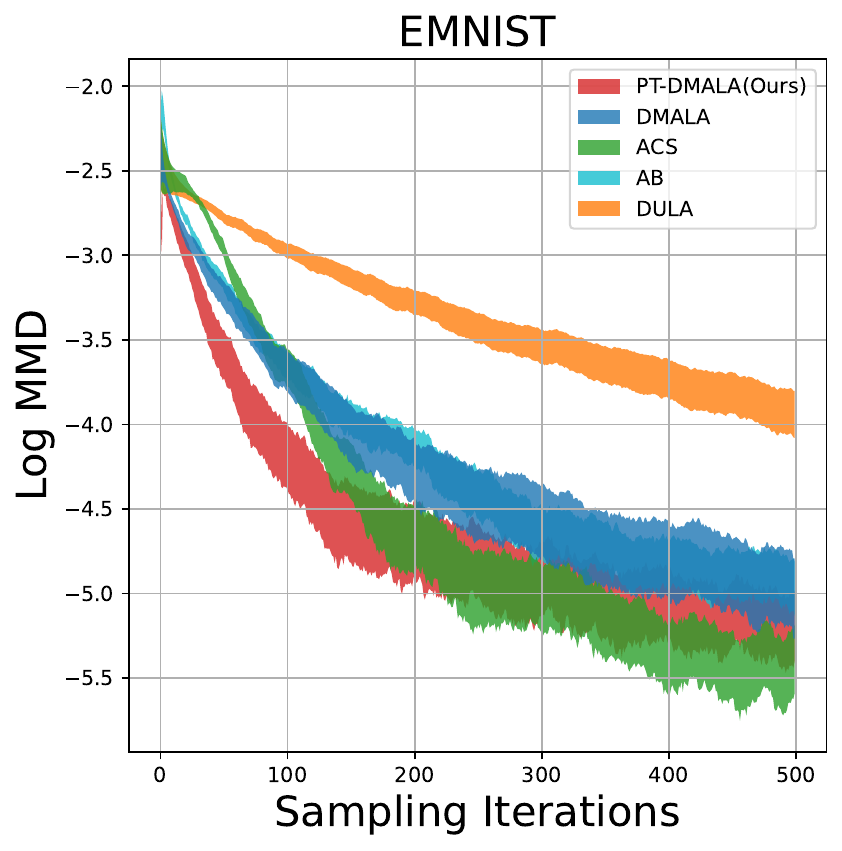}
    \end{minipage}
    \begin{minipage}{0.195\textwidth}
      \includegraphics[width=\textwidth]{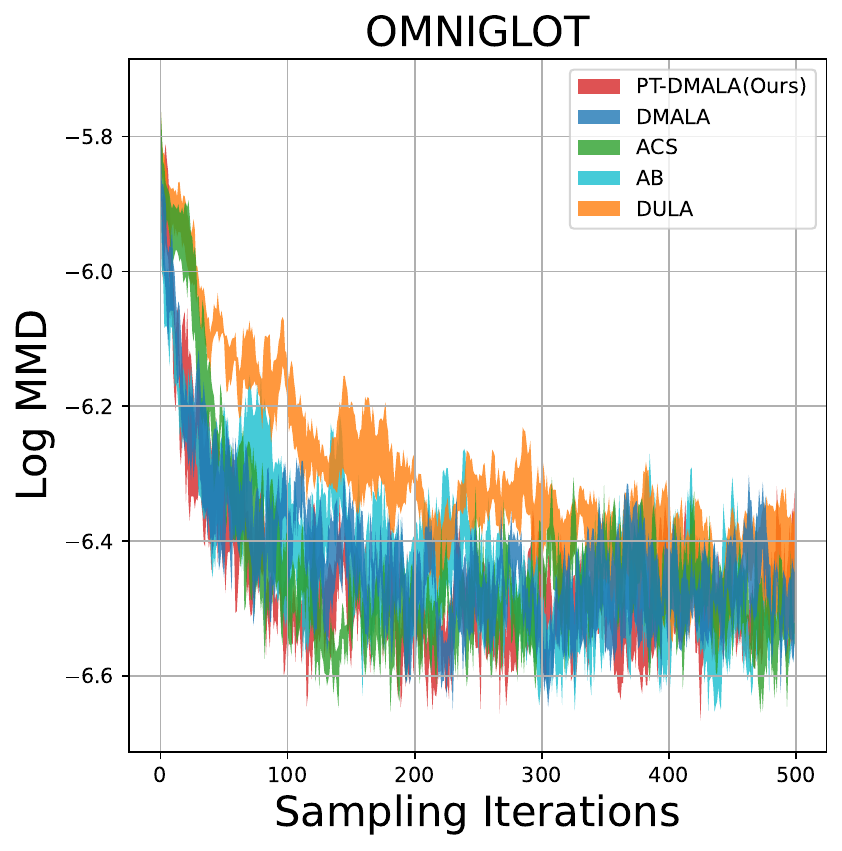}
    \end{minipage} 
    \begin{minipage}{0.195\textwidth}
      \includegraphics[width=\textwidth]{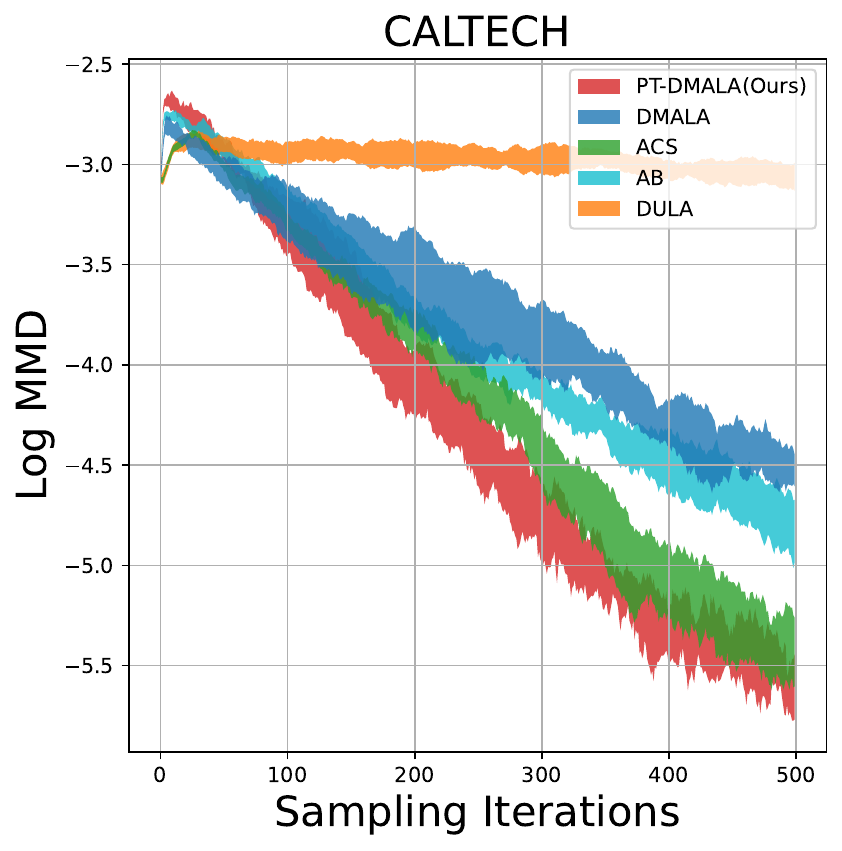}
    \end{minipage}
  \end{center}
  \caption{RBM sampling results with local mode initialization. PT-DMALA achieves faster convergence, while baseline methods converge slower due to being trapped in the mode.}
  \label{fig_mode_int}
\end{figure}

\begin{table*}[!ht]
    \caption{RBM sampling and learning with random initialization. Top table shows $\textit{log-MMD}$ ($\downarrow$) where PT-DMALA outperforms gradient-based baselines. Bottom table presents $\textit{log-likelihood scores}$ ($\uparrow$) for RBM learning, with PT-DMALA showing competitive or superior performance.}
    \label{tab_rbm_sampling_learning}
    \centering
    \resizebox{0.95\textwidth}{!}{
    \begin{tabular}{cccccccc}
    \toprule
        \multicolumn{2}{c}{Dataset}  & DULA & DMALA & AB & ACS & PT-DMALA~\textit{(Ours)}& \\
    \midrule
        \multirow{5}{*}{\makecell{RBM Sampling\\(\textit{log-MMD} $\downarrow$)}} 
        & MNIST & -4.77\fs{0.34} & -6.45\fs{0.29} & -6.65\fs{0.10} & -6.66\fs{0.20} & \textbf{-6.68\fs{0.19}} \\
        & eMNIST & -3.19\fs{0.07} & -3.85\fs{0.10} & -4.01\fs{0.12} & -3.97\fs{0.09} & \textbf{-4.05\fs{0.09}} \\
        & kMNIST & -4.03\fs{0.10} & -4.62\fs{0.17} & -4.71\fs{0.11} & -4.58\fs{0.12} & \textbf{-4.77\fs{0.20}} \\
        & Omniglot & -6.45\fs{0.15} & -6.48\fs{0.08} & -6.48\fs{0.11} & -6.49\fs{0.18} & \textbf{-6.56\fs{0.05}} \\
        & Caltech & -3.09\fs{0.10} & -4.10\fs{0.28} & -4.21\fs{0.33} & -4.21\fs{0.20} & \textbf{-4.98\fs{0.23}} \\
    \midrule
        \multirow{5}{*}{\makecell{Learning RBM\\(\textit{log-likelihood} $\uparrow$)}}
        & MNIST & -386.21\fs{2.32} & -264.83\fs{1.51} & ---& -231.12\fs{2.10} & \textbf{-225.56\fs{2.25}} \\
        & eMNIST & -337.27\fs{4.21} & -324.34\fs{2.13} & ---& \textbf{-301.42\fs{1.99}} & -302.78\fs{1.95} \\
        & kMNIST & -502.22\fs{3.76} & -436.35\fs{2.76} & ---& -407.39\fs{3.46} & \textbf{-362.85\fs{3.97}} \\
        & Omniglot & -228.23\fs{2.12} & -222.61\fs{1.33} &--- & -220.71\fs{1.65} & \textbf{-179.54\fs{2.01}} \\
        & Caltech & -452.97\fs{6.10} & -427.29\fs{3.99} &---& -380.67\fs{3.01} & \textbf{-346.65\fs{2.76}} \\
    \bottomrule 
    \end{tabular}}
\end{table*}
\subsection{Sampling From Restricted Boltzmann Machines}
Restricted Boltzmann Machines (RBMs) are generative stochastic neural networks grounded in probabilistic graphical models~\citep{fischer2012introduction}. We evaluate our method on RBMs trained across a variety of binary datasets. Specifically, RBMs model an unnormalized probability distribution over input data:
$$
U(\theta)=\sum_i \operatorname{Softplus}(W \theta+a)_i+b^{\top} \theta, 
$$
where $\{W, a, b\}$ are parameters and $\theta \in\{0,1\}^d$. Following~\citet{grathwohl2021oops} and~\citet{zhang2022langevin}, we train $\{W, a, b\}$ with contrastive divergence~\citep{carreira2005contrastive} on various datasets. We measure the MMD between the obtained samples and those from block-Gibbs sampling, which utilizes the known structure and can be regarded as the ground truth. To further test whether our method can escape local modes, we initialize all samplers to start
within the most likely mode of the dataset as measured by
the model distribution.

\paragraph{Results and Analysis.} \cref{tab_rbm_sampling_learning} shows that PT-DMALA consistently achieves superior performance across various real-world datasets, especially on Caltech. While AB and ACS perform comparably, they exhibit slightly higher MMD scores. As illustrated in \Cref{fig_mode_int}, our method converges more rapidly across seeds, whereas DULA often collapses to a single mode. These results demonstrate the robustness of PT-DMALA in avoiding mode collapse, a limitation of competing approaches.



\subsection{Learning Energy Based Models}
Energy-based models (EBMs) have achieved notable success in machine learning~\citep{lecun2006tutorial, ngiam2011learning}. In EBMs, the probability of a data point \(x\) is given by 
$P_\theta(x) = \exp\left[E_\theta(x)\right] / Z_\theta$, where \(E_\theta(x)\) is the energy function and $Z_\theta = \mathbb{E}_{\theta \sim \Theta} \left[\exp\left[E_\theta(x)\right]\right]$
is the partition function. MCMC methods are widely used for training EBMs, enabling efficient sampling.

\subsubsection{Learning RBM}
We begin with learning RBM, use the same RBM structure as the sampling task, and apply the samplers of interest to the Persistent Contrastive Divergence (PCD) algorithm introduced by \citet{tieleman2008training}. To evaluate the learned model, we employ Annealed Importance Sampling (AIS)~\citep{neal2001annealed} with Block-Gibbs to calculate the log-likelihood values. We run AIS for 100,000 steps, which is adequate given the efficiency of Block Gibbs for this specific model.
\subsubsection{Learning Deep EBM}
\begin{wraptable}{r}{0.57\textwidth}
\vspace{-10pt}
\caption{Deep Convolution EBM \textit{log likelihood scores}~($\uparrow$) on test data as estimated by AIS.}\label{tab_ebm_learning}
\centering
\resizebox{1\linewidth}{!}{
\begin{tabular}{lcccc}
\toprule
       &  DMALA  & ACS & PT-DMALA~\textit{(Ours)} \\
\midrule
Static MNIST & -80.031\fs{0.038} &-79.905\fs{0.057} &\textbf{-79.622\fs{0.063}}\\
Dynamic MNIST &-80.120\fs{0.036} &-79.634\fs{0.024}&\textbf{-79.463\fs{0.076}}\\
Omniglot & -99.243\fs{2.101} &-91.487\fs{0.128} &\textbf{-90.976\fs{0.316}}\\
Caltech & -98.001\fs{0.371} &-89.262\fs{0.290} &\textbf{-87.192\fs{0.343}}
\\
\bottomrule
\end{tabular}}
\vspace{-10pt}
\end{wraptable}
We train deep EBMs using a ResNet~\citep{he2016deep} with PCD and a replay buffer~\citep{du2019implicit} on the MNIST, Omniglot, and Caltech datasets, following the approach outlined by~\citet{grathwohl2021oops, zhang2022langevin}. We use 10 sample steps per iteration on all datasets except Caltech, where we use 30. After training, we use AIS to estimate the likelihood. 
\paragraph{Results and Analysis.}
In~\cref{tab_rbm_sampling_learning}, we find that our algorithm produces competitive results compared to the baselines, and in many cases outperforms them across all datasets, demonstrating the superiority and robustness of using multiple chains. The results in~\cref{tab_ebm_learning} show that our method is capable of learning better quality EBMs than DMALA and ACS, which can be attributed to the fact that our method employs different temperatures to simultaneously explore diverse regions, enabling the identification of more modes.

In learning tasks, data often originates from complex, high-dimensional distributions. The strong empirical results highlight the effectiveness of the proposed swap mechanism in \Cref{swap_function}, which improves the representativeness of samples, resulting in better log-likelihood estimates.

\section{Conclusions}
In this paper, we propose the \textit{Parallel Tempering enhanced Discrete Langevin Proposal} algorithm to better capture multimodal distributions in discrete spaces. Gradient-based samplers are prone to getting trapped in local modes, hindering full exploration of target distributions. To address this, we incorporate parallel tempering for more effective mode exploration. We also optimize the extra hyperparameters, such as the temperature schedule and the number of chains, by maximizing the round trip rate. Additionally, we establish the asymptotic and non-asymptotic convergence bounds and provide extensive experimental results.

\textbf{Limitations and future work.} While we have developed a reversible algorithm with non-asymptotic guarantees, \citet{syed2022non,deng2023non} demonstrated that non-reversible parallel tempering often outperforms its reversible counterpart. Future work could explore combining non-reversible PT methods with discrete samplers. Another limitation is that we only provide a better guaranteed upper bound on the convergence rate of PT-DMALA compared to DLP. In future work, we aim to further develop theoretical guarantees to quantify the acceleration more precisely.



\bibliographystyle{plainnat}
\bibliography{reference}

\newpage
\appendix

\section{Parallel Tempering Enhanced Discrete Langevin Proposal}
\begin{algorithm}[!]
  \caption{Parallel Tempering Discrete Langevin Proposal~(PTDLP for short).}
  \begin{algorithmic}
    \label{alg:dlp}
    \STATE \textbf{given:} Step size $\boldsymbol{\alpha}$, sampling steps $\hat{n}$, temperatures $\mathcal{T}_K$, chain number $K$, swap intensity $\rho$, initial samples $\{\theta_k(0)\}_{k=1}^{K}\in \Theta^K$
    \LOOP
    \FOR[{{\color{black} Can be done in parallel}}]{$k = 1, \cdots, K$}
      \STATE {\color{blue} Sampling step}
      \FOR[{{\color{black} Can be done in parallel}}]{$i = 1, \cdots, d$}
      \STATE \textbf{construct} $q_{k}(\cdot|\theta)_i$ as in \Cref{factorize}
      \STATE \textbf{sample} $\theta_{k,i}^{\prime} \sim q_k(\cdot|\theta)_i$
      \ENDFOR
    \STATE {\color{blue} M-H step (Optional)}
    \STATE \textbf{compute} $q(\theta'|\theta) = \prod_i q_i(\theta'_i|\theta)$ and $q(\theta|\theta') = \prod_i q_i(\theta_i|\theta')$
    \STATE \textbf{set} $\theta \leftarrow \theta'$ with probability in \Cref{MH}
    \ENDFOR
    \STATE {\color{blue} Swapping step}
    \STATE $\{u_k\}_{k=1}^{K-1}\leftarrow \operatorname{Unif}(0,1)$
    \FOR{$k = 1, \cdots, K-1$}
    \STATE \textbf{construct} $s_k$ as in \cref{swap_function}
    \STATE \textbf{exchange} $\theta_{k}$ and $\theta_{k+1}$ if $u_k \leq \rho\min\{1, s_k\}$
    \ENDFOR
    \ENDLOOP
    \STATE \textbf{output}: Samples $\{\theta_1(n)\}_{n=0}^{\hat{n}}$
  \end{algorithmic}
\end{algorithm}

\section{Algorithm with Different Variables}
\paragraph{Binary Variables.}
When the variable domain $\Theta$ is binary, i.e., $\{0,1\}^d$, the algorithm in \cref{alg:dlp} can be further simplified for each chain update, which clearly shows that our method can be efficiently computed in parallel on both CPUs and GPUs.
\begin{algorithm}[H]
  \caption{Each chain with Binary Variables}
  \begin{algorithmic}
    \label{alg:binary}
    \STATE \textbf{given:} Stepsize $\alpha$, sampling steps $\hat{n}$, initial samples $\theta_0$
      \LOOP
      \STATE \textbf{compute} $p(\theta) = \frac{\exp(-\frac{1}{2}\nabla U(\theta)\odot (2\theta-1) - \frac{1}{2\alpha})}{\exp(-\frac{1}{2}\nabla U(\theta)\odot (2\theta-1) - \frac{1}{2\alpha})+1}$
      \STATE \textbf{sample} $\mu \sim \text{Unif}(0,1)^d$
      \STATE $I \leftarrow \texttt{dim}(\mu\le p(\theta))$
    \STATE $\theta'\leftarrow \texttt{flipdim}(I)$
    \STATE \textbf{compute} $q(\theta'|\theta) = \prod_i q_i(\theta'_i|\theta)=\prod_{i\in I} p(\theta)_i \cdot \prod_{i\notin I} (1-p(\theta)_i)$
    \STATE \textbf{compute} $p(\theta') = \frac{\exp(-\frac{1}{2}\nabla U(\theta')\odot (2\theta'-1) - \frac{1}{2\alpha})}{\exp(-\frac{1}{2}\nabla U(\theta')\odot (2\theta'-1) - \frac{1}{2\alpha})+1}$
    \STATE \textbf{compute} $q(\theta|\theta') = \prod_i q_i(\theta_i|\theta')=\prod_{i\in I} p(\theta')_i \cdot \prod_{i\notin I} (1-p(\theta')_i)$
    \STATE \textbf{set} $\theta \leftarrow \theta'$ with probability in \cref{MH}
    \ENDLOOP
    \STATE \textbf{output}: Samples $\{\theta_n\}_{n=0}^{\hat{n}}$
  \end{algorithmic}
\end{algorithm}

\paragraph{Category Variables.}
When using one-hot vectors (for unordered categories) and standard categorical variables (with a clear ordering) to represent categorical data, our discrete Langevin proposal becomes
$$
\text { Categorical }\left(\operatorname{Softmax}\left(\frac{\beta}{2} \nabla U(\theta)_i^{\top}\left(\theta_i^{\prime}-\theta_i\right)-\frac{\left\|\theta_i^{\prime}-\theta_i\right\|_p^p}{2 \alpha}\right)\right).
$$
\section{Iterative Tuning Algorithm}\label{tuning_algo}

\begin{algorithm}[h]
  \caption{Iterative Tuning Algorithm}
  \begin{algorithmic}
    \label{alg:ITA}
    \STATE \textbf{given:} Initial temperature schedule $\mathcal{T}_{K}$ of size $K$, tuning steps $n_{\max}$, sampling steps $\hat{n}$, and $\epsilon$
    \FOR{$n = 1, \cdots, n_{\max}$}
    \STATE $\left\{\hat{s}_k\right\}_{k=1}^{K-1} \leftarrow$ PTDLP $\left(\mathcal{T}_{K}, \hat{n}\right)$
    \STATE \textbf{calculate} points $\{(\beta_1, \hat{\Lambda}_n(\beta_1)), \cdots, (\beta_K, \hat{\Lambda}_n(\beta_K))\}$ 
    \STATE \textbf{compute} $\hat{\Lambda}(\cdot)$ by using monotone piecewise cubic interpolation~\citep{fritsch1980monotone}
        \FOR{$k=1, \cdots, K$}
        \STATE \textbf{find} $\beta_k^*$ using~\citet[Eq.(30)]{syed2022non} and bisection
        \STATE $\beta_k \leftarrow \beta_k^{\ast}$  
        \ENDFOR
    \IF{$\left|\hat{\Lambda}_n - \hat{\Lambda}_{n-1}\right| < \epsilon$}
        \STATE $\hat{\Lambda}_{n_{\max}}(\cdot) \leftarrow \hat{\Lambda}_{n}(\cdot)$
        \STATE \textbf{break}
    \ENDIF
    \ENDFOR
    \STATE $K^{\ast}\! \leftarrow 2\hat{\Lambda}_{n_{\max}} + 1$
    \FOR{$k=1, \cdots, K^{\ast}$}
        \STATE \textbf{find} $\beta_k^*$ using~\citet[Eq.(30)]{syed2022non} and bisection
    \ENDFOR
    \STATE \textbf{output}: Optimal temperature schedule $\mathcal{T}^{\ast}_{K^{\ast}}$ and chain number $K^{\ast}$
    
  \end{algorithmic}
\end{algorithm}

\section{Technical Appendices and Supplementary Material}
\subsection{Proofs in \cref{sec_4}}\label{app_D_1}
\begin{lemma}[\citet{nadler2007dynamics,syed2022non}]\label{round_trip_rate}
For any fixed chain number $K$ and temperature schedule $\mathcal{T}_K=\left\{\beta_1, \ldots, \beta_K\right\}$, the non-asymptotic~(in $K$) round trip rate of the reversible PT scheme with all neighboring chains swapping is 
$$
\tau\left(\mathcal{T}_K\right)=\frac{1}{\sum_{k=1}^{K-1}\left(1 / s_k\right)}, 
$$
where $s_k$, defined in \Cref{swap_function}, is the probability of swapping between chains $k$ and $k+1$.
\end{lemma}
\begin{proof}[\textbf{Proof of \cref{lemma_4_2}}]
Recall that the round trip rate of our algorithm with $\mathcal{B}$ copies is 
\begin{equation}\label{tau_Lambda}
\tau_{\mathcal{B}}(K)=\frac{\mathcal{B}}{\sum_{k=1}^{K-1}1 / s_k}.
\end{equation}
By using the fact that the swap rates are all equal and~\citet[Corollary 2]{syed2022non}, we obtain, for any $k=1, \cdots, K-1$, $\sum_{k=1}^{K-1} \frac{1}{s_k} = \frac{K-1}{1 - \Lambda / (K-1)}$. Substituting the above equation into \Cref{tau_Lambda} yields
\begin{equation}\label{tao_Lambda_1}
\tau_{\mathcal{B}} = \frac{\mathcal{B}(K-1-\Lambda)}{(K-1)^2}.
\end{equation}
To maximize \cref{tao_Lambda_1}, we need to take the derivative and find the critical points. Denote by $f(K):=\frac{(K-1-\Lambda)}{(K-1)^2}$ and let $f^{\prime}(K^{\ast})=0$, we obtain, 
$$
K^{\ast}=2 \Lambda + 1.
$$
Finally, by verifying the second derivative, we determine that this point corresponds to a maximum.

\end{proof}

\subsection{Proofs in \cref{sec_5_1}}
\begin{lemma}\label{lm:1}
    If the target distribution is assumed to be log-quadratic, i.e., for any $\theta \in \Theta$, $\pi^{\beta_k}(\theta) \propto \exp \left(\beta_k(\theta^{\top} W \theta+b^{\top} \theta)\right)$ with some constants\footnote{ Without loss of generality, we assume $W$ is symmetric, otherwise we can replace $W$ with $\left(W+W^{\top}\right) / 2$ for the eigendecomposition.} $W \in \mathbb{R}^{d \times d}$ and $b \in \mathbb{R}^d$. Then the Markov chain following transition $q_{k}(\cdot \mid \theta)$ in \Cref{PTDLP_2} for any $k=1, \cdots, K$ is reversible with respect to some distribution $\pi_\alpha^{\beta_k}$, and $\pi_\alpha^{\beta_k}$ converges weakly to $\pi^{\beta_k}$ as $\alpha \rightarrow 0$. 
\end{lemma}
\begin{proof}
    The main idea of the proof is to replace the gradient term in the proposal by the energy difference $U(\theta^{\prime}) - U(\theta)$ using Taylor series approximation, and then show the reversibility of the chain based on the proofs in \citet{zanella2020informed, zhang2022langevin}. We divide the proof into two parts. In the first part, we prove the convergence of our algorithm to $\pi_{\alpha}^{\beta_k}$, and in the second part, we derive the distance between $\pi_{\alpha}^{\beta_k}$ and $\pi_{\alpha}$.
    
    Recall that the target distribution is $\pi^{\beta_k}(\theta)=\exp \left(\beta_k(\theta^{\top} W \theta+b^{\top} \theta)\right) / Z$. We have that $\nabla \log(\pi(\theta))=\beta_k(2 W^{\top} \theta+b)$, $\nabla^2 \log(\pi(\theta))=2\beta_k W$. Then, by using the fact that $U\left(\theta^{\prime}\right)-U(\theta)=\nabla U(\theta)^{\top}\left(\theta^{\prime}-\theta\right)+\frac{1}{2}\left(\theta^{\prime}-\theta\right)^{\top} 2 W\left(\theta^{\prime}-\theta\right)$ by Taylor series approximation, we can rewrite the proposal distribution as the following
\begin{equation}\label{log_quatratic}
\begin{aligned}
&q_k\left(\theta^{\prime} \mid \theta\right)
\\& =\frac{\exp \left(\frac{\beta_k}{2} \nabla U(\theta)^{\top}\left(\theta^{\prime}-\theta\right)+\frac{\beta_k}{2}\left(\theta^{\prime}-\theta\right)^{\top} W\left(\theta^{\prime}-\theta\right)- \frac{\beta_k}{2} \left(\theta^{\prime}-\theta\right)^{\top} W\left(\theta^{\prime}-\theta\right) - \frac{1}{2\alpha}||\theta^{\prime} - \theta||_p^p\right)}{\sum_x \exp \left(\frac{\beta_k}{2} \nabla U(\theta)^{\top}(x-\theta)+\frac{\beta_k}{2}(x-\theta)^{\top} W(x-\theta)- \frac{\beta_k}{2}(x-\theta)^{\top} W(x-\theta) - \frac{1}{2\alpha}||\theta^{\prime} - \theta||_p^p\right)} \\
& =\frac{\exp \left(\frac{\beta_k}{2}\left(U\left(\theta^{\prime}\right)-U(\theta)\right)-\frac{\beta_k}{2}\left(\theta^{\prime}-\theta\right)^{\top} W\left(\theta^{\prime}-\theta\right) - \frac{1}{2\alpha}||\theta^{\prime} - \theta||_p^p\right)}{\sum_x \exp \left(\frac{\beta_k}{2}(U(x)-U(\theta))-\frac{\beta_k}{2}(x-\theta)^{\top} W(x-\theta) - \frac{1}{2\alpha}||\theta^{\prime} - \theta||_p^p\right)}.
\end{aligned}
\end{equation}
Debote by $Z_\alpha^{\beta_k}(\theta)=\sum_x \exp \left(\frac{\beta_k}{2}(U(x)-U(\theta))-\frac{\beta_k}{2}(x-\theta)^{\top} W(x-\theta) - \frac{1}{2\alpha}||\theta^{\prime} - \theta||_p^p\right)$, and $\pi_\alpha^{\beta_k}=\frac{Z_\alpha^{\beta_k}(\theta) \pi^{\beta_k}(\theta)}{\sum_x Z_\alpha^{\beta_k}(x) \pi^{\beta_k}(x)}$, now we will show that $q_k$ is reversible w.r.t. $\pi_\alpha^{\beta_k}$. We have that
\begin{equation}\label{reverse}
\begin{aligned}
&\pi_\alpha^{\beta_k}(\theta) q_k\left(\theta^{\prime} \mid \theta\right)\\ & =\frac{Z_\alpha^{\beta_k}(\theta) \pi^{\beta_k}(\theta)}{\sum_x Z_\alpha^{\beta_k}(x) \pi^{\beta_k}(x)} \frac{\exp \left(\frac{\beta_k}{2}\left(U\left(\theta^{\prime}\right)-U(\theta)\right)-\frac{\beta_k}{2}\left(\theta^{\prime}-\theta\right)^{\top} W\left(\theta^{\prime}-\theta\right)- \frac{1}{2\alpha}||\theta^{\prime} - \theta||_p^p\right)}{Z_\alpha^{\beta_k}(\theta)} \\
& =\frac{\exp \left(\frac{\beta_k}{2}\left(U\left(\theta^{\prime}\right)+U(\theta)\right)-\frac{\beta_k}{2}\left(\theta^{\prime}-\theta\right)^{\top} W\left(\theta^{\prime}-\theta\right) - \frac{1}{2\alpha}||\theta^{\prime} - \theta||_p^p\right)}{Z^{\beta_k} \cdot \sum_x Z_\alpha^{\beta}(x) \pi^{\beta_k}(x)}.
\end{aligned}
\end{equation}
We note that \cref{reverse} is symmetric in $\theta$ and $\theta^{\prime}$. Therefore $q_k$ is reversible and its stationary distribution is $\pi_\alpha^{\beta_k}(\theta)$. Next, we will prove that $\pi_\alpha^{\beta_k}$ converges weakly to $\pi^{\beta_k}$ as $\alpha \rightarrow 0$. Notice that for any $\theta$,
$$
\begin{aligned}
Z_\alpha^{\beta_k}(\theta) & =\sum_x \exp \left(\frac{\beta_k}{2}(U(x)-U(\theta))-\frac{\beta_k}{2}(x-\theta)^{\top} W (x-\theta)-\frac{1}{2\alpha}\left\|\theta^{\prime} - \theta\right\|_p^p\right)\\&
\stackrel{\alpha \downarrow 0}=1
\end{aligned}
$$
By using Scheffé's Lemma, we have that $\pi_\alpha$ converges weakly to $\pi$.
\end{proof}

\begin{proof}[\textbf{Proof of \cref{thm_5_1}}]
    To explore the reversibility of our algorithm, we extend the proof of \citet{zhang2022langevin}. We first consider the transition probability of the first chain~(with temperature equals to 1) $q_{\alpha}(\theta^{\prime} | \theta_i^{(1)})$ in our algorithm. Considering the presence of the swap mechanism, we write
\small{
\begin{equation}\label{q_first_chain}
\begin{aligned}
&q_{\alpha}\left( {\theta}^{\prime} \mid  {\theta}_i^{(1)}\right)\\&= \sum_{ {\theta}_i^{(2)}} \sum_{ {\theta}_{i+1}^{(2)}} \pi_{\alpha}^{\beta_2}\left( {\theta}_i^{(2)}\right) q_2\left( {\theta}_{i+1}^{(2)} \mid  {\theta}_i^{(2)}\right)\left[1-s_{1}\left( {\theta}^{\prime},  {\theta}_{i+1}^{(2)}\right)\right] q_1\left( {\theta}^{\prime} \mid  {\theta}_i^{(1)}\right) \\
& \quad +\!\sum_{ {\theta}_i^{(3)}}\sum_{ {\theta}_{i+1}^{(3)}}\sum_{ {\theta}_i^{(2)}} \sum_{ {\theta}_{i+1}^{(1)}} \pi_{\alpha}^{\beta_2}\!\left( {\theta}_i^{(2)}\right) q_{2}\left( {\theta}^{\prime} \mid  {\theta}_i^{(2)}\right) s_{1}\left( {\theta}_{i+1}^{(1)},  {\theta}^{\prime}\right) q_{1}\left( {\theta}_{i+1}^{(1)} \mid  {\theta}_i^{(1)}\right)\\
&\quad \times\left(1 \!-\! s_2(\theta^{\prime}, \theta_{i+1}^{(3)})\right)q_{ 3}\left(\theta_{i+1}^{(3)}|\theta_i^{(3)}\right)\pi_{\alpha}^{\beta_3}(\theta_i^{(3)})+\!\sum_{ {\theta}_i^{(3)}}\sum_{ {\theta}_{i}^{(2)}}\sum_{ {\theta}_{i+1}^{(2)}} \sum_{ {\theta}_{i+1}^{(1)}} \pi_{\alpha}^{\beta_2}\left( {\theta}_i^{(2)}\right)\\
&\quad \times q_{2}\left( {\theta}_{i+1}^{(2)} \mid  {\theta}_i^{(2)}\right) s_{1}\left( {\theta}_{i+1}^{(1)},  {\theta}^{\prime}\right) q_{1}\left( {\theta}_{i+1}^{(1)} \mid  {\theta}_i^{(1)}\right)s_2\left(\theta_{i+1}^{(2)}, \theta^{\prime}\right)q_{3}\left(\theta^{\prime}|\theta_i^{(3)}\right)\pi_{\alpha}^{\beta_3}(\theta_i^{(3)}).
\end{aligned}
\end{equation}}
To demonstrate the reversibility, we multiply $\pi_\alpha^{\beta_1}(\theta)$ from both sides:
\begin{align*}
& \pi_\alpha^{\beta_1}\left( {\theta}_i^{(1)}\right) q_{\alpha}\left( {\theta}^{\prime} \mid  {\theta}_i^{(1)}\right) \\
& =\sum_{ {\theta}_i^{(2)}} \sum_{ {\theta}_{i+1}^{(2)}} \pi_\alpha^{\beta_2}\left( {\theta}_i^{(2)}\right) q_{2}\left( {\theta}_{i+1}^{(2)} \mid  {\theta}_i^{(2)}\right)\left[1-s_1\left( {\theta}^{\prime},  {\theta}_{i+1}^{(2)}\right)\right] \pi_\alpha^{\beta_1}\left( {\theta}_i^{(1)}\right) q_{1}\left( {\theta}^{\prime} \mid  {\theta}_i^{(1)}\right) \\
& \quad +\!\sum_{ {\theta}_i^{(3)}}\sum_{ {\theta}_{i+1}^{(3)}}\sum_{ {\theta}_i^{(2)}} \sum_{ {\theta}_{i+1}^{(1)}} \bigg(\pi_{\alpha}^{\beta_2}\!\left( {\theta}_i^{(2)}\right) q_{2}\left( {\theta}^{\prime} \mid  {\theta}_i^{(2)}\right) s_{1}\left( {\theta}_{i+1}^{(1)},  {\theta}^{\prime}\right) \pi_\alpha^{\beta_1}\left( {\theta}_i^{(1)}\right) q_{1}\left( {\theta}_{i+1}^{(1)} \mid  {\theta}_i^{(1)}\right)\\ & \quad \times \left(1 \!-\! s_2(\theta^{\prime}, \theta_{i+1}^{(3)})\right)q_{3}\left(\theta_{i+1}^{(3)}|\theta_i^{(3)}\right)\pi_{\alpha}^{\beta_3}(\theta_i^{(3)})\bigg)\\
&\quad +\!\sum_{ {\theta}_i^{(3)}}\sum_{ {\theta}_{i}^{(2)}}\sum_{ {\theta}_{i+1}^{(2)}} \sum_{ {\theta}_{i+1}^{(1)}} \pi_{\alpha}^{\beta_2}\!\left( {\theta}_i^{(2)}\right) q_{2}\left( {\theta}_{i+1}^{(2)} \mid  {\theta}_i^{(2)}\right) s_{1}\left( {\theta}_{i+1}^{(1)},  {\theta}^{\prime}\right) \pi_\alpha^{\beta_1}\left( {\theta}_i^{(1)}\right) q_{1}\left( {\theta}_{i+1}^{(1)} \mid  {\theta}_i^{(1)}\right) \\ & \quad \times s_2\left(\theta_{i+1}^{(2)}, \theta^{\prime}\right)q_{3}\left(\theta^{\prime}|\theta_i^{(3)}\right)\pi_{\alpha}^{\beta_3}(\theta_i^{(3)}),
\end{align*}
where $s_1$ and $s_2$ are defined in \Cref{swap_function}. Note that, by using \cref{lm:1}, $\pi_\alpha^{\beta_1}\left(\theta_i^{(1)}\right) q_{\alpha, 1}\left(\cdot \mid \theta_i^{(1)}\right)$, $\pi_\alpha^{\beta_2}\left(\theta_i^{(2)}\right) q_{\alpha, 2}\left(\cdot \mid \theta_i^{(2)}\right)$, and $\pi_\alpha^{\beta_3}\left(\theta_i^{(3)}\right) q_{\alpha, 3}\left(\cdot \mid \theta_i^{(3)}\right)$ are symmetric, which indicates that $\pi_\alpha\left( {\theta}_i^{(1)}\right) q_{\alpha}\left( {\theta}^{\prime} \mid  {\theta}_i^{(1)}\right)$ is also symmetric. Therefore, we conclude that $q_{\alpha}\left( {\theta}^{\prime} \mid  {\theta}_i^{(1)}\right)$ in~\cref{q_first_chain} is reversible and the stationary distribution is $\pi_\alpha^{\beta_1}\left( {\theta}_i^{(1)}\right)$. Next, to generalize the convergence result from log-quadratic distributions to general distributions, we assume that $\exists W \in$ $\mathbb{R}^{d \times d}, b \in \mathbb{R}, \epsilon \in \mathbb{R}^{+}$, such that
$$
\|\nabla U(\theta)-(2 W \theta+b)\|_1 \leq \epsilon, \forall \theta \in \Theta .
$$
Then, recall that $\pi^{\beta_1}$ is the target distribution, $\tilde{\pi}^{\beta_1}$ is the log-quadratic distribution that is close to $\pi^{\beta_1}$, $\pi_{\alpha}^{\beta_1}$ is the stationary distribution of our algorithm without M-H step. $\tilde{\pi}_{\alpha}^{\beta_1}$ is the stationary distribution of our algorithm targeting $\tilde{\pi}^{\beta_1}$. By using~\citet[Theorem 5.2]{zhang2022langevin} and~\citet[Proposition 4.2]{levin2017markov}, we obtain 
$$
\begin{aligned}
\left\|\pi_\alpha-\pi\right\|_{TV} &\leq \left\|\pi_\alpha - \tilde{\pi}_{\alpha}\right\|_{TV} + \ \left\|\tilde{\pi}_{\alpha} - \tilde{\pi}\right\|_{TV} + \left\| \tilde{\pi}-\pi\right\|_{TV}\\
&\leq Z_1\left(\exp(Z_2\epsilon\right) + Z_3 \exp \left(-\frac{1+\alpha \beta_1\lambda_{\min }}{2 \alpha}\right) - Z_1.
\end{aligned}
$$

\end{proof}

\begin{proof}[\textbf{Proof of \cref{thm:5.2}}]
First, we consider the lower bound of the mixing time. Recall~\cref{q_first_chain}, we have
\small{
\begin{equation}\label{sum_q_alpha}
\begin{aligned}
&\sum_{\theta^{\prime}\neq \theta_i^{(1)}}q_{\alpha}\left( {\theta}^{\prime} \mid  {\theta}_i^{(1)}\right)
\\&= \sum_{\theta^{\prime}\neq \theta_i^{(1)}}\sum_{ {\theta}_i^{(2)}} \sum_{ {\theta}_{i+1}^{(2)}} \pi_{\alpha}^{\beta_2}\left( {\theta}_i^{(2)}\right) q_{2}\left( {\theta}_{i+1}^{(2)} \mid  {\theta}_i^{(2)}\right)\left[1-s_{1}\left( {\theta}^{\prime},  {\theta}_{i+1}^{(2)}\right)\right] q_{1}\left( {\theta}^{\prime} \mid  {\theta}_i^{(1)}\right) \\
& \quad +\!\sum_{\theta^{\prime}\neq \theta_i^{(1)}}\!\sum_{ {\theta}_i^{(3)}}\sum_{ {\theta}_{i+1}^{(3)}}\sum_{ {\theta}_i^{(2)}} \sum_{ {\theta}_{i+1}^{(1)}} \pi_{\alpha}^{\beta_2}\!\left( {\theta}_i^{(2)}\right) q_{2}\left( {\theta}^{\prime} \!\mid\!  {\theta}_i^{(2)}\right) s_{1}\left( {\theta}_{i+1}^{(1)},  {\theta}^{\prime}\right) q_{1}\left( {\theta}_{i+1}^{(1)} \!\mid\!  {\theta}_i^{(1)}\right)\\
&\quad \times \left(\!1 \!-\! s_2(\theta^{\prime}, \theta_{i+1}^{(3)})\!\right)q_{3}\left(\theta_{i+1}^{(3)}|\theta_i^{(3)}\right)\pi_{\alpha}^{\beta_3}(\theta_i^{(3)}) + \sum_{\theta^{\prime}\neq \theta_i^{(1)}}\sum_{ {\theta}_i^{(3)}}\sum_{ {\theta}_{i}^{(2)}}\sum_{ {\theta}_{i+1}^{(2)}} \sum_{ {\theta}_{i+1}^{(1)}} \pi_{\alpha}^{\beta_2}\left( {\theta}_i^{(2)}\right) \\
&\quad \times q_{2}\left( {\theta}_{i+1}^{(2)} \mid  {\theta}_i^{(2)}\right) s_{1}\left( {\theta}_{i+1}^{(1)},  {\theta}^{\prime}\right) q_{1}\left( {\theta}_{i+1}^{(1)} \mid  {\theta}_i^{(1)}\right)s_2\left(\theta_{i+1}^{(2)}, \theta^{\prime}\right)q_{3}\left(\theta^{\prime}|\theta_i^{(3)}\right)\pi_{\alpha}^{\beta_3}(\theta_i^{(3)})\\
&\leq \sum_{\theta^{\prime}\neq \theta_i^{(1)}}\left(q_{1}\left( {\theta}^{\prime} \mid  {\theta}_i^{(1)}\right) + \sum_{ {\theta}_i^{(2)}} \pi_{\alpha}^{\beta_2}\!\left( {\theta}_i^{(2)}\right) q_{2}\left( {\theta}^{\prime} \mid  {\theta}_i^{(2)}\right) + \sum_{ {\theta}_i^{(3)}}  \pi_{\alpha}^{\beta_3}(\theta_i^{(3)}) q_{3}\left(\theta^{\prime}|\theta_i^{(3)}\right)\right)\\
&\leq \sum_{\theta^{\prime}\neq \theta_i^{(1)}}q_{1}\left( {\theta}^{\prime} \mid  {\theta}_i^{(1)}\right) + 2 . 
\end{aligned}
\end{equation}}
By using~\cref{log_quatratic} and letting $\beta_1 = 1$, we have
\begin{equation}\label{q_1_leq}
\begin{aligned}
q_{1}\left(\theta^{\prime} \mid \theta\right)
& =\frac{\exp \left(\frac{1}{2}\left(U\left(\theta^{\prime}\right)-U(\theta)\right)-\frac{1}{2}\left(\theta^{\prime}-\theta\right)^{\top} W\left(\theta^{\prime}-\theta\right) - \frac{1}{2\alpha}||\theta^{\prime} - \theta||_p^p\right)}{\sum_x \exp \left(\frac{1}{2}(U(x)-U(\theta))-\frac{1}{2}(x-\theta)^{\top} W(x-\theta) - \frac{1}{2\alpha}||\theta^{\prime} - \theta||_p^p\right)}\\
&= \frac{\exp \left(\frac{1}{2}\left(U\left(\theta^{\prime}\right)-U(\theta)\right)-\frac{1}{2}\left(\theta^{\prime}-\theta\right)^{\top} W\left(\theta^{\prime}-\theta\right) - \frac{1}{2\alpha}||\theta^{\prime} - \theta||_p^p\right)}{1 + \sum_{x \neq \theta} \exp \left(\frac{1}{2}(U(x)-U(\theta))-\frac{1}{2}(x-\theta)^{\top} W(x-\theta) - \frac{1}{2\alpha}||\theta^{\prime} - \theta||_p^p\right)}\\
&\leq \exp \left\{\frac{1}{2}\left(U\left(\theta^{\prime}\right)-U(\theta)\right)-\frac{1}{2}\left(\theta^{\prime}-\theta\right)^{\top} W\left(\theta^{\prime}-\theta\right) - \frac{1}{2\alpha}||\theta^{\prime} - \theta||_p^p\right\}.
\end{aligned}
\end{equation}
By substituting~\cref{q_1_leq} into~\cref{sum_q_alpha} and the fact that $\frac{{x}^{\top} W {x}}{{x}^{\top} {x}} \geq \lambda_{\min }(W)$ for any $x\neq 0$, one writes
$$
\begin{aligned}
&\sum_{\theta^{\prime}\neq \theta_i^{(1)}}q_{\alpha}\left( {\theta}^{\prime} \mid  {\theta}_i^{(1)}\right) \\&\leq \sum_{\theta^{\prime}\neq \theta_i^{(1)}}\exp \left\{\frac{1}{2}\left(U\left(\theta^{\prime}\right)-U(\theta_i^{(1)})\right)-\frac{1}{2}\left(\theta^{\prime}-\theta_i^{(1)}\right)^{\top} W\left(\theta^{\prime}-\theta_i^{(1)}\right) - \frac{1}{2\alpha}||\theta^{\prime} - \theta_i^{(1)}||_p^p\right\} + 2\\
&\leq 2 \sum_{\theta^{\prime}}\exp \left\{\frac{1}{2}\left(U\left(\theta^{\prime}\right)-U(\theta_i^{(1)})\right)-\frac{1}{2}\left(\theta^{\prime}-\theta_i^{(1)}\right)^{\top} W\left(\theta^{\prime}-\theta_i^{(1)}\right) - \frac{1}{2\alpha}||\theta^{\prime} - \theta_i^{(1)}||_p^p\right\}\\
&\leq 2 \sum_{\theta^{\prime}}\exp \left\{\frac{1}{2}\left(U\left(\theta^{\prime}\right)-U(\theta_i^{(1)})\right)-\frac{1}{2}\lambda_{\min}(W)d_2 - \frac{1}{2\alpha}d_p\right\}\\
&\leq 2 \exp \left\{-\frac{1}{2}\lambda_{\min}(W)d_2 - \frac{1}{2\alpha}d_p\right\}\sum_{\theta^{\prime}}\exp \left\{\frac{1}{2}\left(U\left(\theta^{\prime}\right)-U(\theta_i^{(1)})\right)\right\}\\
&\leq  2 Z \exp \left\{-\frac{1}{2}\lambda_{\min}(W)d_2 - \frac{1}{2\alpha}d_p\right\},
\end{aligned}
$$
where $Z$ is the normalizing constant of the target distribution $\pi$. Note that $q_\alpha$ is reversible and the transition matrix of a reversible Markov chain has only real eigenvalues. By using~\citet[Theorem 6.1.1]{horn2012matrix}, there at least exists one $\theta \in \Theta$ such that
$$
\left|\lambda_2-q_\alpha(\theta \mid \theta)\right| \leq Z \exp \left(-\frac{1}{2} \lambda_{\min }(W) d_2-\frac{1}{2 \alpha} d_p\right),
$$
where $\lambda_2$ is the second largest eigenvalue of the transition matrix. Then we consider the spectral gap~\citep[Chaper 12]{levin2017markov},
\begin{equation}\label{lambda_2}
\begin{aligned}
1-\lambda_2 & \leq\left|1-q_\alpha(\theta \mid \theta)\right|+\left|q_\alpha(\theta \mid\theta)-\lambda_2\right| \\
& \leq\left|1-q_\alpha(\theta \mid\theta)\right| + 2Z \exp \left(-\frac{1}{2} \lambda_{\min }(W) d_2-\frac{1}{2 \alpha} d_p\right) \\
& =\sum_{\theta^{\prime} \neq \theta} q_\alpha(\theta^{\prime} \mid \theta)+2Z \exp \left(-\frac{1}{2} \lambda_{\min}(W) d_2-\frac{1}{2 \alpha} d_p\right) \\
& \leq 4 \cdot Z \exp \left(-\frac{1}{2} \lambda_{\min }(W) d_2-\frac{1}{2 \alpha} d_p\right) .
\end{aligned}
\end{equation}
Denote by $t_{\operatorname{mix}}(\varepsilon):=\min \{t: d(t) \leq \varepsilon\}$ with $d(t):=\max _{\theta \in \Theta}\left\|P_\alpha(\theta, \cdot)-\pi_\alpha\right\|_{\mathrm{TV}}$. By using~\citet[Theorem 12.7]{levin2017markov} and~\cref{lambda_2}, we obtain
$$
\begin{aligned}
t_{\operatorname{mix}}(\varepsilon) &\geq\left(\frac{1}{1-\lambda_2}-1\right) \log \left(\frac{1}{2 \varepsilon}\right)\\
&\geq \left(\frac{1}{4 \cdot Z}\exp \left(\frac{1}{2} \lambda_{\min }(W) d_2+\frac{1}{2 \alpha} d_p\right)-1\right) \log \left(\frac{1}{2 \varepsilon}\right) :=\mathcal{L}.
\end{aligned}
$$
Then we consider to find the upper bound of the mixing time. Our proof idea is to analyze the conductance of the algorithm and apply the Cheeger inequality to derive a lower bound on $1 - \lambda_2$. First, we denote the conductance of the chain by
$$\Phi:= \min_{S:0<\pi_{\alpha}(S)\leq 1 / 2} \frac{Q(S, S^{c})}{\pi_{\alpha}(S)},$$ 
where $Q(S, S^{c}):=\sum_{\theta\in S,\theta^{\prime}\in S^{c}} \pi_{\alpha}(\theta)q_{\alpha}(\theta^{\prime}|\theta)$ is the probability flow, with $\pi_\alpha:=\pi_\alpha^{\beta_1}$ for simplicity. $\Phi$ measures the relative width of the most difficult "bottleneck" in the state space; the larger $\Phi$ is, the faster the mixing. We next aim to establish a positive lower bound for $\Phi$. We assume that $|U(\cdot)|\leq U_{\max}$. Recall that $q_\alpha\left(\theta^{\prime} \mid \theta\right)$ given by~\cref{log_quatratic} (we choose $\beta_1=1$), by using $\frac{{x}^{\top} W {x}}{{x}^{\top} {x}} \leq \lambda_{\max}(W)$ for any $x\neq 0$, the numerator can be writen as 
\begin{equation}\label{thm_5_2_eq_1}
\exp \left(\frac{1}{2}\left(U\left(\theta^{\prime}\right)-U(\theta)\right)-\frac{1}{2}\left(\theta^{\prime}-\theta\right)^{\top} \!W\left(\theta^{\prime}-\theta\right) - \frac{1}{2\alpha}||\theta^{\prime} - \theta||_p^p\right) \geq \exp\{-U_{\max} - \frac{1}{2}\lambda_{\max}\mathcal{D}_2 - \frac{1}{2\alpha}\mathcal{D}_p\}.
\end{equation}
The denominator can be rescaled as
\begin{equation}\label{thm_5_2_eq_2}
\begin{aligned}
&\sum_x \exp \left(\frac{1}{2}(U(x)-U(\theta))-\frac{1}{2}(x-\theta)^{\top} W(x-\theta) - \frac{1}{2\alpha}||\theta^{\prime} - \theta||_p^p\right) \\&
\leq 1+ \sum_{x\neq \theta}\exp\{U_{\max} - \frac{1}{2}\lambda_{\min}d_2 -\frac{1}{2\alpha}d_p\}:=D_{\max}.
\end{aligned}
\end{equation}
By combining~\cref{thm_5_2_eq_1,thm_5_2_eq_2}, we arrive at
\begin{equation}\label{q_min}
q_{\alpha, 1}(\theta^{\prime}|\theta)\geq\frac{\exp\{-U_{\max} - \frac{1}{2}\lambda_{\max}\mathcal{D}_2 - \frac{1}{2\alpha}\mathcal{D}_p\}}{D_{\max}}:=q_{\min, 1}.
\end{equation}
Then, we obtain
\begin{equation}\label{thm_5_2_eq3}
\begin{aligned}
q_{\alpha}(\theta^{\prime}|\theta_i^{(1)})&\geq \sum_{ {\theta}_i^{(2)}} \sum_{ {\theta}_{i+1}^{(2)}} \pi_{\alpha}^{\beta_2}\left( {\theta}_i^{(2)}\right) q_{2}\left( {\theta}_{i+1}^{(2)} \mid  {\theta}_i^{(2)}\right)\left[1-s_{1}\left( {\theta}^{\prime},  {\theta}_{i+1}^{(2)}\right)\right] q_{1}\left( {\theta}^{\prime} \mid  {\theta}_i^{(1)}\right) 
\\
& \quad +\sum_{ {\theta}_i^{(3)}}\sum_{ {\theta}_{i+1}^{(3)}}\sum_{ {\theta}_i^{(2)}} \sum_{ {\theta}_{i+1}^{(1)}} \pi_{\alpha}^{\beta_2}\left( {\theta}_i^{(2)}\right)  q_{2}\left( {\theta}^{\prime} \mid  {\theta}_i^{(2)}\right) s_{1}\left( {\theta}_{i+1}^{(1)},  {\theta}^{\prime}\right) q_{1}\left( {\theta}_{i+1}^{(1)} \mid  {\theta}_i^{(1)}\right)\\
& \quad \times\left(1 - s_2(\theta^{\prime}, \theta_{i+1}^{(3)})\right)q_{3}\left(\theta_{i+1}^{(3)}|\theta_i^{(3)}\right)\pi_{\alpha}^{\beta_3}(\theta_i^{(3)}) +\!\sum_{ {\theta}_i^{(3)}}\sum_{ {\theta}_{i}^{(2)}}\sum_{ {\theta}_{i+1}^{(2)}} \sum_{ {\theta}_{i+1}^{(1)}} \pi_{\alpha}^{\beta_2}\left( {\theta}_i^{(2)}\right) \\
& \quad \times q_{2}\left( {\theta}_{i+1}^{(2)} \mid  {\theta}_i^{(2)}\right) s_{1}\left( {\theta}_{i+1}^{(1)},  {\theta}^{\prime}\right) q_{1}\left( {\theta}_{i+1}^{(1)} \mid  {\theta}_i^{(1)}\right)s_2\left(\theta_{i+1}^{(2)}, \theta^{\prime}\right)q_{3}\left(\theta^{\prime}|\theta_i^{(3)}\right)\pi_{\alpha}^{\beta_3}(\theta_i^{(3)})
\\&\geq q_{\min,1} \underbrace{\left( \sum_{\theta_i^{(2)}, \theta_{i+1}^{(2)}} \pi_\alpha^{\beta_2}(\theta_i^{(2)}) q_{\alpha,2}(\theta_{i+1}^{(2)} | \theta_i^{(2)}) [1 - s_1(\theta', \theta_{i+1}^{(2)})] \right)}_{\mathbb{E}\left[1 - s_1\right]}:=q_{\min}
\end{aligned}
\end{equation}
Note that the expectation $\mathbb{E}[1 - s_1]$ represents the average probability that no swap occurs between chain 1 (at proposed state \( \theta' \)) and chain 2 (after one of its own updates). Since \( 0 < s_1 \leq 1 \), it follows that $q_{\max} > 0$. Denote by $\partial(S, S^{c})$ the boundary, where $\theta \in S$, $\theta^{\prime} \in S^{c}$, and $\theta, \theta^{\prime}$ are neighbors. Assume that the transition happens primarily occur between neighbors, we have
$$
Q(S, S^{c}) = \sum_{\theta\in S, \theta^{\prime}\in S^{c}}\pi_{\alpha}q_{\alpha}(\theta^{\prime}|\theta) = \sum_{(\theta,\theta^{\prime})\in \partial(S, S^{c})}\pi_{\alpha}q_{\alpha}(\theta^{\prime}|\theta).
$$
We assume that $\pi_{\alpha} \geq \pi_{\alpha,\min}$ for any $\theta \in \Theta$. By using~\Cref{thm_5_2_eq3}, we have
$$
Q(S, S^{c}) \geq |\partial (S, S^{c})|\pi_{\alpha,\min}\ q_{\min},
$$
where $ |\partial (S, S^{c})|$ is the number of neighboring pairs crossing the cut. Thus,
\begin{equation}\label{thm_5_2_eq_4}
\Phi = \min_{S:\pi_{\alpha}\leq 1 / 2} \frac{Q(S,S^{c})}{\pi_{\alpha}(S)} \geq I_{\pi_{\alpha}}(\Theta) \pi_{\alpha,\min}\ q_{\min},
\end{equation}
where $I_{\pi_{\alpha}}(\Theta):=\min_{S:\pi_{\alpha}\leq 1 / 2}\frac{|\partial (S, S^{c})|}{\pi_{\alpha}(S)}$. By using the fact that $1- \lambda_2 \geq \frac{\Phi^2}{2}$,~\Cref{thm_5_2_eq_4}, and~\citet[Theorem 12.5]{levin2017markov}, we obtain
$$
t_{\operatorname{mix}}(\varepsilon)\leq \frac{1}{1 - \lambda_2}\log(\frac{1}{\epsilon\pi_{\alpha, \min}})\leq \frac{2}{\left(I_{\pi_{\alpha}}(\Theta) \pi_{\alpha,\min}\ q_{\min}\right)^2}\log(\frac{1}{\epsilon\pi_{\alpha,\min}}):=\mathcal{U}.
$$


\end{proof}

\section{Proofs in \cref{sec_5_2}}
We define the problem setting in more detail. For any $k=1,\cdots, K$, we define
$$
\pi^{\beta_k}(\theta)=\frac{1}{Z} \exp (\beta_kU(\theta)).
$$
We consider the proposal kernel as, for $k=1, \cdots, K$,
$$
q_k\left(\theta^{\prime} \mid \theta\right) \propto \exp \left\{\beta_k \nabla U(\theta)^\top\left(\theta^{\prime}-\theta\right)-\frac{1}{2 \alpha}\left\|\theta^{\prime}-\theta\right\|_p^p\right\}, 
$$
and consider the transition kernel as
$$
\hat{q}_k\left(\theta^{\prime} \mid \theta\right)=\left(\frac{\pi^{\beta_k}\left(\theta^{\prime}\right) q_k\left(\theta \mid \theta^{\prime}\right)}{\pi^{\beta_k}(\theta) q_k\left(\theta^{\prime} \mid \theta\right)} \wedge 1\right) q_k\left(\theta^{\prime} \mid \theta\right)+(1-L(\theta)) \delta_\theta\left(\theta^{\prime}\right),
$$
where
$$
L(\theta)=\sum_{\theta^{\prime} \in \Theta}\min\left\{\frac{\pi^{\beta_k}\left(\theta^{\prime}\right) q_k\left(\theta \mid \theta^{\prime}\right)}{\pi^{\beta_k}(\theta) q_k\left(\theta^{\prime} \mid \theta\right)}, 1\right\} q_k\left(\theta^{\prime} \mid \theta\right)
$$
is the total rejection probability from $\theta$. Finally, recall that the total variation distance between two probability measures $\mu$ and $\nu$, defined on some space $\Theta \subset \mathbb{R}^d$ is
$$
\|\mu-\nu\|_{T V}=\sup _{A \in \mathcal{B}(\Theta)}|\mu(A)-\nu(A)|,
$$
where $\mathcal{B}(\Theta)$ is the set of all measurable sets in $\Theta$. We have the following assumptions:
\begin{assumption}\label{asm_2}
The function \(U(\cdot)\in C^{2}(\mathbb{R}^{d})\) has \(M\)-Lipschitz gradient. That is
\[
\|\nabla U(\theta)-\nabla U(\theta^{\prime})\|\leq M\|\theta-\theta^{\prime}\|\,.\]
\end{assumption}
\begin{assumption}\label{asm_3}
For each \(\theta\in\Theta\), there exists an open ball containing \(\theta\) of some radius \(r_{\theta}\), denoted by \(R(\theta,r_{\theta})\), such that the function \(U(\cdot)\) is \(m\)-strongly concave in \(R(\theta,r_{\theta})\) for some \(m>0\).
\end{assumption}
\cref{asm_2,asm_3} are standard in optimization and sampling literature~\citep{dalalyan2017theoretical}. 

\begin{lemma}[\citet{pynadath2024gradient}]\label{lm:3}
Let \cref{asm_2,asm_3} hold. Then we have, for any $k=1, \cdots, K$ and $\theta$, $\theta^{\prime} \in \Theta$,
$$
\hat{q}_k(\theta^{\prime}\mid\theta)\geq\epsilon_{\beta_k,\alpha}\,\frac{\exp\left\{\beta_k U (\theta^{\prime})\right\}}{\sum_{\theta^{\prime}\in\Theta}\exp\left\{\beta_k U( \theta^{\prime})\right\}},
$$
where
$$
\epsilon_{\beta_k,\alpha}=\exp\left\{-\beta_k\left(M-\frac{m}{2}\right)\mathcal{D}_2-\beta_k\|\nabla U(a)\|\,\mathcal{D}_1 - \frac{1}{\alpha}\mathcal{D}_p\right\}, 
$$
with $a\in\arg\min_{\theta\in\Theta}\|\nabla U(\theta)\|$.
\end{lemma}

\begin{proof}[\textbf{Proof of \cref{thm:4}}]
For brevity, we take three chains as an example. We denote the transition kernel of PT-DMALA by
\begin{equation}\label{pt_dmala_p}
\begin{aligned}
&p\left( {\theta}^{\prime} \mid  {\theta}_i^{(1)}\right)\\&= \sum_{ {\theta}_i^{(2)}} \sum_{ {\theta}_{i+1}^{(2)}} \pi^{\beta_2}\left( {\theta}_i^{(2)}\right) \hat{q}_2\left( {\theta}_{i+1}^{(2)} \mid  {\theta}_i^{(2)}\right)\left[1-s_{1}\left( {\theta}^{\prime},  {\theta}_{i+1}^{(2)}\right)\right] \hat{q}_1\left( {\theta}^{\prime} \mid  {\theta}_i^{(1)}\right) \\
& \quad +\!\sum_{ {\theta}_i^{(3)}}\sum_{ {\theta}_{i+1}^{(3)}}\sum_{ {\theta}_i^{(2)}} \sum_{ {\theta}_{i+1}^{(1)}} \pi^{\beta_2}\!\left( {\theta}_i^{(2)}\right) \hat{q}_2\left( {\theta}^{\prime} \mid  {\theta}_i^{(2)}\right) s_{1}\left( {\theta}_{i+1}^{(1)},  {\theta}^{\prime}\right) \hat{q}_1\left( {\theta}_{i+1}^{(1)} \mid  {\theta}_i^{(1)}\right)\\
&\quad \times \left(1 \!-\! s_2(\theta^{\prime}, \theta_{i+1}^{(3)})\right)\hat{q}_3\left(\theta_{i+1}^{(3)}|\theta_i^{(3)}\right)\pi^{\beta_3}(\theta_i^{(3)}) + \!\sum_{ {\theta}_i^{(3)}}\sum_{ {\theta}_{i}^{(2)}}\sum_{ {\theta}_{i+1}^{(2)}} \sum_{ {\theta}_{i+1}^{(1)}} \pi^{\beta_2}\left( {\theta}_i^{(2)}\right) \\&\quad \times q^{\beta_2}\left( {\theta}_{i+1}^{(2)}\!\mid\!  {\theta}_i^{(2)}\right) s_{1}\left( {\theta}_{i+1}^{(1)},  {\theta}^{\prime}\right) \hat{q}_1\left( {\theta}_{i+1}^{(1)}\!\mid\!  {\theta}_i^{(1)}\right)s_2\left(\theta_{i+1}^{(2)}, \theta^{\prime}\right)\hat{q}_3\left(\theta^{\prime}|\theta_i^{(3)}\right)\pi^{\beta_3}(\theta_i^{(3)}).
\end{aligned}
\end{equation}
Since the state space is finite, for any $\theta\in\Theta$, we can denote the maximum and minimum values of $U(\theta)$  as $u_{\max}$ and $u_{\min}$, respectively. Therefore, for any $k=1, \cdots, K-1$, we obtain
$$
\begin{aligned}
s_k\left(\theta^{(k)}_{i+1},\theta^{(k+1)}_{i+1}\ |\ \theta^{(1)}_{i},\theta ^{(2)}_{i}\right)&=e^{\left(\beta_{k}-\beta_{k+1}\right)\left[U( \theta^{(k+1)}_{i+1})+U(\theta^{(k+1)}_{i})-U(\theta^{(k)}_{i+1})-U(\theta^{(k)}_{i })\right]}
\\& \geq e^{-2\left(\beta_{k}-\beta_{k+1}\right)\left(u_{\max}-u_{\min}\right)}.
\end{aligned}
$$
Denote by 
\begin{equation}\label{epsilon_0}
\epsilon_{0}:=\min\limits_{k=1, \cdots, K-1}\left\{\exp\left\{-2\left(\beta_{k}-\beta_{k+1}\right)\left(u_{\max}-u_{\min}\right)\right\}\right\}.
\end{equation}
Then, for any $k=1, \cdots, K-1$, we can get that, $1\geq s_k\geq\epsilon_{0}> 0$. By using \Cref{epsilon_0,lm:3}, we obtain
\begin{align*}
&p\left( {\theta}^{\prime} \mid  {\theta}_i^{(1)}\right)\\&= \sum_{ {\theta}_i^{(2)}} \sum_{ {\theta}_{i+1}^{(2)}} \pi^{\beta_2}\left( {\theta}_i^{(2)}\right) \hat{q}_2\left( {\theta}_{i+1}^{(2)} \mid  {\theta}_i^{(2)}\right)\left[1-s_{1}\left( {\theta}^{\prime},  {\theta}_{i+1}^{(2)}\right)\right] \hat{q}_1\left( {\theta}^{\prime} \mid  {\theta}_i^{(1)}\right) \\
& \quad +\!\sum_{ {\theta}_i^{(3)}}\sum_{ {\theta}_{i+1}^{(3)}}\sum_{ {\theta}_i^{(2)}} \sum_{ {\theta}_{i+1}^{(1)}} \pi^{\beta_2}\!\left( {\theta}_i^{(2)}\right) \hat{q}_2\left( {\theta}^{\prime} \mid  {\theta}_i^{(2)}\right) s_{1}\left( {\theta}_{i+1}^{(1)},  {\theta}^{\prime}\right) \hat{q}_1\left( {\theta}_{i+1}^{(1)} \mid  {\theta}_i^{(1)}\right)\\
&\quad \times \left(1 \!-\! s_2(\theta^{\prime}, \theta_{i+1}^{(3)})\right)\hat{q}_3\left(\theta_{i+1}^{(3)}|\theta_i^{(3)}\right)\pi^{\beta_3}(\theta_i^{(3)}) + \!\sum_{ {\theta}_i^{(3)}}\sum_{ {\theta}_{i}^{(2)}}\sum_{ {\theta}_{i+1}^{(2)}} \sum_{ {\theta}_{i+1}^{(1)}} \pi^{\beta_2}\left( {\theta}_i^{(2)}\right) \\&\quad \times \hat{q}_2\left( {\theta}_{i+1}^{(2)}\!\mid\!  {\theta}_i^{(2)}\right) s_{1}\left( {\theta}_{i+1}^{(1)},  {\theta}^{\prime}\right) \hat{q}_1\left( {\theta}_{i+1}^{(1)}\!\mid\!  {\theta}_i^{(1)}\right)s_2\left(\theta_{i+1}^{(2)}, \theta^{\prime}\right)\hat{q}_3\left(\theta^{\prime}|\theta_i^{(3)}\right)\pi^{\beta_3}(\theta_i^{(3)})\\
&\geq \epsilon_{0}^2\left(\sum_{ {\theta}_{i}^{(2)}}\sum_{ {\theta}_{i+1}^{(2)}} \pi^{\beta_2}\!\left( {\theta}_i^{(2)}\right) \hat{q}_2\left( {\theta}_{i+1}^{(2)} \mid  {\theta}_i^{(2)}\right) \right)\left(\sum_{ {\theta}_i^{(3)}}\hat{q}_3\left(\theta^{\prime}|\theta_i^{(3)}\right)\pi^{\beta_3}(\theta_i^{(3)})\right)\left(\sum_{ {\theta}_{i+1}^{(1)}}\hat{q}_1\left( {\theta}_{i+1}^{(1)} \mid  {\theta}_i^{(1)}\right)\right)\\
&=\epsilon_{0}^2\epsilon_{\beta_3, \alpha}\frac{\exp\left\{\beta_{3} U (\theta^{\prime})\right\}}{\sum_{\theta^{\prime}\in\Theta}\exp\left\{\beta_{3} U(\theta^{\prime})\right\}},
\end{align*}
where 
\begin{equation}\label{epsilon_3}
\epsilon_{\beta_3, \alpha} := \exp\left\{ -\beta_3\left( M - \frac{m}{2} \right) \mathcal{D}_2 
- \beta_3 \| \nabla U(a) \| \mathcal{D}_1 
- \frac{1}{\alpha} \mathcal{D}_p \right\},
\end{equation}
where \(a \in \arg\min_{\theta \in \Theta} \| \nabla U(\theta) \|\), $M$ and $m$ are from~\Cref{asm_2,asm_3}. It then follows from~\citet[Corollary 5]{jones2004markov} that the chain is uniformly ergodic.

\end{proof}

\begin{proof}[\textbf{Proof of \cref{col:5.7}}]
By using \cref{thm:4} and the fact $0<\epsilon<1$, we note that as the $\epsilon$ approaches $1$, the sampling algorithm converges faster in terms of total variance. Specifically, we consider $$k:=\frac{\epsilon_0^2 \epsilon_{\beta_3, \alpha}}{\epsilon_{\beta_1, \alpha}}= \frac{\epsilon_0^2\exp\left\{-\beta_3\left((M-\frac{m}{2})\mathcal{D}_2 - \|\nabla U(a)\| \mathcal{D}_1\right)-\frac{1}{\alpha}\mathcal{D}_p\right\}}{\exp\left\{-\beta_1\left((M-\frac{m}{2})\mathcal{D}_2 - \|\nabla U(a)\| \mathcal{D}_1\right)-\frac{1}{\alpha}\mathcal{D}_p\right\}}.$$ 
By using the definition of $\mathcal{D}_p$ and $\|\nabla U(a)\| < \left((M-\frac{m}{2})\mathcal{D}_2 - \log(1 / \epsilon_0^2)\right)/ \mathcal{D}_1 $ and the fact that $\beta_3 < \beta_1$, we obtain
$$
\begin{aligned}
k& = \epsilon_0^2 \exp\left\{\left(\beta_1-\beta_3\right) \left(\left(M-\frac{m}{2}\right)\mathcal{D}_2 - \|\nabla U(a)\|\mathcal{D}_1\right)\right\}\\ &
> 1 . 
\end{aligned}
$$
Thus, we could conclude that PT-DMALA provides a better guaranteed upper bound on convergence speed compared to DLP.
\end{proof}

\section{Additional Experiments Results and Setting Details}\label{App_ex}
This section complements the main text by providing details on additional experimental procedures.
\subsection{Sampling from Synthetic Energies}
\paragraph{Synthetic Distribution.}
We examine  energy functions with varying components of MoG and MoS, and use forward KL, MMD, and EMC to evaluate the algorithms' capability to navigate through complex terrains with multiple local minima and discontinuities.

The probability density function of \textbf{MoG} is given by:
\[
p_{\text{MoG}}(\mathbf{x}) = \sum_{k=1}^{K_1} \pi_k \cdot \mathcal{N}(\mathbf{x} \mid \boldsymbol{\mu}_k, \boldsymbol{\Sigma}_k),
\]
where $K_1$ denotes the number of Gaussian components, $\pi_k$ denotes the mixing weight of the $k$-th component satisfying $\sum_{k=1}^{K_1} \pi_k = 1$ and $\pi_k > 0$, and $\mathcal{N}(\mathbf{x} \mid \boldsymbol{\mu}_k, \boldsymbol{\Sigma}_k)$ denotes the probability density function of a d-dimensional Gaussian distribution. 

Similarly, the probability density function of \textbf{MoS} is given by:
\[
p_{\text{MoS}}(\mathbf{x}) = \sum_{j=1}^{K_2} \pi_j \cdot t(\mathbf{x} \mid \boldsymbol{\mu}_j, \boldsymbol{\Sigma}_j, \nu_j),
\]
where $K_2$ denotes the number of Student's t components, $\pi_j$ denotes the mixing weight of the $j$-th component satisfying $\sum_{j=1}^{K_2} \pi_j = 1$ and $\pi_j > 0$, and $t(\mathbf{x} \mid \boldsymbol{\mu}_j, \boldsymbol{\Sigma}_j, \nu_j)$ denotes the probability density function of a  d-dimensional Student's t-distribution.

\textbf{EMC} is the expected entropy of the auxiliary distribution, that is,
\[
\text{EMC} := \mathbb{E}^{q_\theta} \mathcal{H} \big( p(\xi \mid \mathbf{x}) \big)
\approx -\frac{1}{N} \sum_{\mathbf{x} \sim q_\theta} \sum_{i=1}^M p(\xi_i \mid \mathbf{x}) \log_M p(\xi_i \mid \mathbf{x}),
\]
where $N$ denotes the number of samples drawn from $q_\theta$.

\textbf{MMD} is a kernel-based test used to compare distributions which is computed as:
\[
\text{MMD}^2(\pi, \tilde{\pi}) = \mathbb{E}_{\mathbf{x}, \mathbf{x}' \sim \pi} \big[k(\mathbf{x}, \mathbf{x}')\big] + 
\mathbb{E}_{\mathbf{y}, \mathbf{y}' \sim \tilde{\pi}} \big[k(\mathbf{y}, \mathbf{y}')\big] - 
2\mathbb{E}_{\mathbf{x} \sim \pi, \mathbf{y} \sim \tilde{\pi}} \big[k(\mathbf{x}, \mathbf{y})\big],
\]
where $k(\mathbf{x}, \mathbf{y})$ is a positive-definite kernel function. MMD measures the similarity between the empirical distributions of the generated and target samples. In practice, however, directly computing MMD is computationally expensive. Therefore, we use an
approximation based on Random Fourier Features (RFF). The feature mapping is defined as:
\[
\phi(X) = \sqrt{\frac{2}{D}} \cos(WX^\top + b),
\]
where $W \in \mathbb{R}^{D \times d}$ are random Gaussian variables sampled from $\mathcal{N}(0, 1 / \sigma^2)$, and $\mathbf{b}$ are random uniform variables in the range $[0, 2\pi]$. The parameter $\sigma$ controls the kernel bandwidth, and $D$ is the number of random features. For two distributions $\pi$ and $\tilde{\pi}$, the empirical mean feature embeddings for $X \sim \pi$ and $Y \sim \tilde{\pi}$ are computed for both distributions:
\[
\bm{\mu}_X = \frac{1}{n} \sum_{i=1}^n \phi(X_i), \quad 
\bm{\mu}_Y = \frac{1}{m} \sum_{i=1}^m \phi(Y_i).
\]
The following approach allows us to efficiently compute the MMD between two distributions using RFF:
\[
\text{MMD}^2(\pi, \tilde{\pi}) \approx \|\bm{\mu}_X - \bm{\mu}_Y\|^2.
\]
\textbf{KL divergence} measures the difference between two probability distributions. Given the distributions $\pi$ and $\tilde{\pi}$, the KL divergence is defined as:

$$
D_{\text{KL}}(\pi \parallel \tilde{\pi}) = \sum_{\theta \in \Theta} \pi(\theta) \log \left( \frac{\pi(\theta)}{\tilde{\pi}(\theta)} \right),
$$

where $\pi(\theta)$ represents the probability of $\theta$ under the target distribution $\pi$, and $\tilde{\pi}(\theta)$ represents the probability of $\theta$ under the empirical distribution $\tilde{\pi}$ obtained from sampling. This metric quantifies the information loss incurred when approximating the target distribution $\pi$ using the empirical distribution $\tilde{\pi}$. A lower value of the KL divergence indicates better performance in approximating the target distribution.

\begin{table}[htbp]
\centering
\caption{MMD~($10^{-3}$)($\downarrow$) results~(MoG) across different components~(c denotes the number of components)}
\begin{tabular}{lcccc}
\toprule
\textbf{Task} & \textbf{DMALA} & \textbf{ACS} & \textbf{AB} & \textbf{\textbf{PT-DMALA~\textit{(Ours)}}} \\
\midrule
c=2 & 0.824 \fs{0.026} & 0.481 \fs{0.021} & 0.479 \fs{0.029} & \textbf{0.229 \fs{0.027}} \\
c=4 & 0.942 \fs{0.023} & 0.694 \fs{0.026} & 0.642 \fs{0.012} & \textbf{0.301 \fs{0.012}} \\
c=6 & 1.076 \fs{0.045} & 0.844 \fs{0.033} & 0.801 \fs{0.023}  & \textbf{0.481 \fs{0.025}} \\
c=8 & 1.214 \fs{0.058} & 0.984 \fs{0.031} & 0.891 \fs{0.026}& \textbf{0.534 \fs{0.015}} \\
c=10 & 1.475 \fs{0.039} & 1.199 \fs{0.028} & 1.101 \fs{0.031} & \textbf{0.592 \fs{0.019}} \\
c=12 & 1.689 \fs{0.043} &  1.304  \fs{0.039} & 1.368\fs{0.022} & \textbf{0.702 \fs{0.022}} \\
c=14 & 1.948 \fs{0.051} & 1.694 \fs{0.047} & 1.621 \fs{0.040} & \textbf{0.815 \fs{0.024}} \\
c=16 & 2.130\fs{0.064} & 1.806 \fs{0.056}  & 1.691 \fs{0.042}  & \textbf{0.824 \fs{0.031}} \\
\bottomrule
\end{tabular}
\end{table}

\begin{table}[htbp]
\centering
\caption{MMD~($10^{-3}$)($\downarrow$) results (MoS) across different components~(c denotes the number of components)}
\begin{tabular}{lcccc}
\toprule
\textbf{Task} & \textbf{DMALA} & \textbf{ACS} & \textbf{AB} & \textbf{\textbf{PT-DMALA~\textit{(Ours)}}} \\
\midrule
c=2 & 0.910 \fs{0.034} & 0.663 \fs{0.016} & 0.596 \fs{0.023} & \textbf{0.291 \fs{0.021}} \\
c=4 &  1.014 \fs{0.048} & 0.701 \fs{0.019} & 0.766 \fs{0.017} & \textbf{0.337 \fs{0.018}} \\
c=6 & 1.319 \fs{0.037} & 1.056 \fs{0.021} & 0.992 \fs{0.031} & \textbf{0.564 \fs{0.019}} \\
c=8 & 1.617\fs{0.051} & 1.406 \fs{0.047}& 1.305 \fs{0.044}& \textbf{0.744\fs{0.028}} \\
c=10 & 1.730 \fs{0.061} & 1.598 \fs{0.041} & 1.708 \fs{0.048} & \textbf{0.824 \fs{0.029}} \\
c=12 & 1.934 \fs{0.054} & 1.894 \fs{0.030} & 1.881 \fs{0.037} & \textbf{0.879 \fs{0.033}} \\
c=14 & 2.095 \fs{0.069} & 2.001 \fs{0.052} & 2.003 \fs{0.043} & \textbf{0.921 \fs{0.027}} \\
c=16 & 2.158 \fs{0.073}& 1.813 \fs{0.061}& 1.515 \fs{0.068} & \textbf{0.941 \fs {0.022}} \\
\bottomrule
\end{tabular}
\end{table}

\paragraph{Sampler Configuration.}
DMALA  is implemented with a step size of $0.15$. For AB, the parameters are set to $\sigma = 0.1$ and $\alpha = 0.5$. ACS employs a cyclical step size scheduler with an initial step size of $0.6$ over $10$ cycles.  For PT-DMALA, based on a pilot run, we determined that the optimal number of chains for this task is between 2 and 5, and accordingly set the temperatures for each chain based on the corresponding results. we set the step size to $0.2$ for all chains.

\paragraph{Results.}
When the number of components in both MoG and MoS varies, our sampler consistently achieves significantly superior KL, MMD, and EMC scores compared to DMALA, ACS, and AB. Its capacity to effectively distribute samples, even in scenarios characterized by disconnected modes and steep energy barriers, underscores its robustness in navigating intricate discrete energy landscapes.


\subsection{RBM Sampling}\label{rbm_sampling_app}
\paragraph{RBM Introduction.}
We will give a brief introduction of the Block-Gibbs sampler used to represent the ground truth of the RBM distribution. For a more in-depth explanation, see~\citet{grathwohl2021oops}. Given the hidden units $h$ and the sample $\theta$, we define the RBM distribution as follows:
$$
\log p(\theta, h)=h^\top W \theta + b^\top \theta+c^\top-\log Z. 
$$
As before, Z is the normalizing constant for the distribution. The sample $x$ is represented by the visible layer with units corresponding to the sample space dimension and $h$ represents the model capacity. It can be shown that the marginal distributions are as follows:
$$
\begin{gathered}
p(x \mid h)=\operatorname{Bernoulli}(W x+c), \\
p(h \mid x)=\operatorname{Bernoulli}\left(W^\top h+b\right).
\end{gathered}
$$
The Block-Gibbs sampler updates $\theta$ and $h$ alternatively, allowing for many of the coordinates to get changed at the same time, due to utilizing the specific structure of the RBM model.

\paragraph{Experiment Setup.}
We follow the experimental setup of~\citet{zhang2022langevin}, using RBM models with 500 hidden units and 784 visible units. We adopt a similar training protocol, training the model for 1,000 iterations. For the mode initialization experiment, we train the model for one epoch to facilitate a better comparison of the results.

\paragraph{Sampler Configuration.}
For DMALA, we set step size to $0.2$ , and for AB we use the default hyperparameters for the first order sampler.
For ACS, we use $\rho^*=0.5$, $\beta_{\max }=0.95$, $\zeta=0.5$, cycle length $s=20$ for all the datasets. We also fix the total overhead of the tuning algorithm to $10 \%$ of the total sampling steps. For PT-DMALA, we set the step size to $0.15 \sim 0.45$ for all chains.

\paragraph{Escape from Local Modes.}
In addition to using the same initialization as~\citet{grathwohl2021oops, zhang2022langevin}, we extend the experiment to measure the ability of a sampler to escape from local modes. We initialize the sampler within the most likely mode, as measured by unnormalized energy of the RBM. Samplers that are less prone to getting trapped in local modes will be able to converge quicker to the ground truth, as measured by $\log$ MMD. We include the performance of the various samplers across 5 random seeds in~\cref{fig_mode_int}. PT-DMALA demonstrates superior robustness to mode-specific initialization due to its capability to escape from local modes.

\paragraph{Generated Images.}
We found that a visual inspection of the generated images demonstrates the ability of PTDLP to escape local modes. To ensure a fair comparison of algorithm performance, we use the same settings and baseline figures as in \citet{pynadath2024gradient}, and include the generated images in \Cref{img_rbm_sampling_app}.
\begin{figure}[htbp]
\centering
\resizebox{\textwidth}{!}{
\subfloat[DULA]{\includegraphics[width=2.8cm]{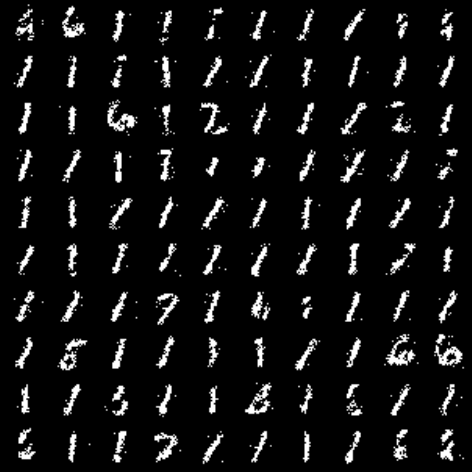}}\hspace*{0pt}
\subfloat[DMALA]{\includegraphics[width=2.8cm]{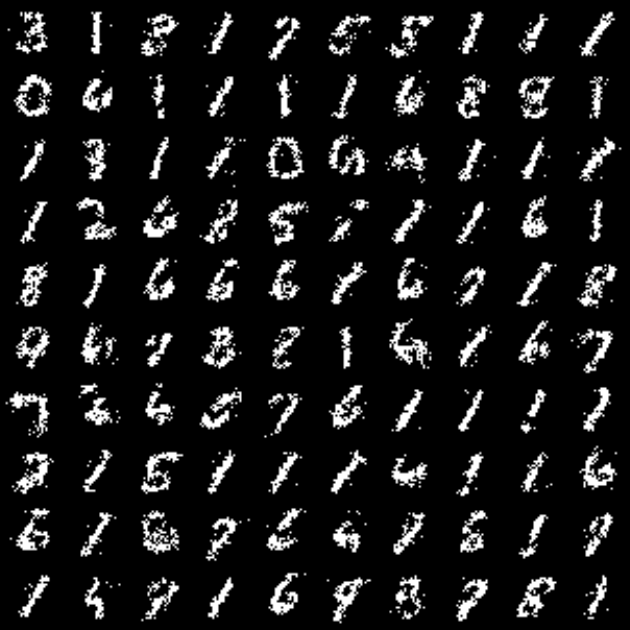}}\hspace*{0pt}
\subfloat[AB]{\includegraphics[width=2.8cm]{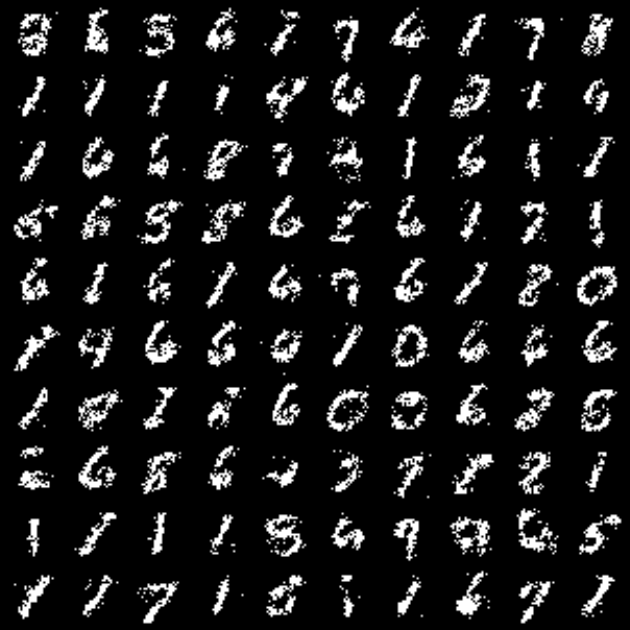}}\hspace*{0pt}
\subfloat[ACS]{\includegraphics[width=2.8cm]{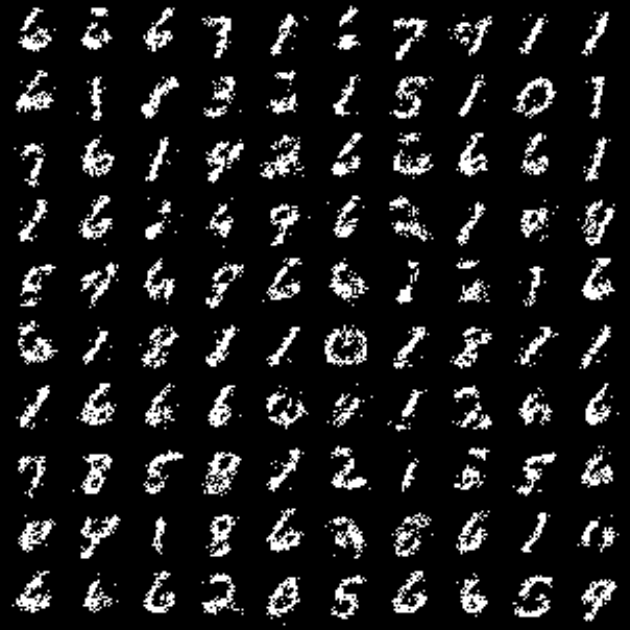}}\hspace*{0pt}
\subfloat[PTDLP]{\includegraphics[width=2.8cm]{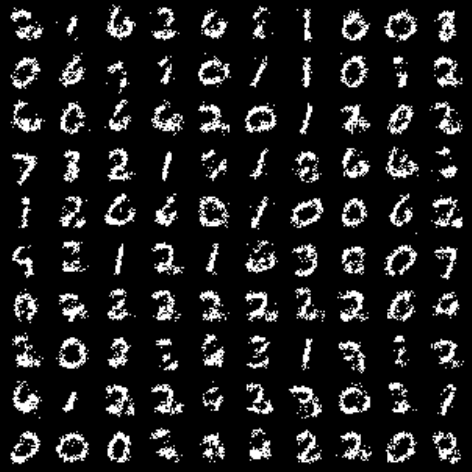}}}
\caption{Images sampled from RBM trained on MNIST when the sampler is initialized to most likely mode. Our algorithm is able to generate a diverse range of digits, demonstrating its ability to escape from modes.}
\label{img_rbm_sampling_app}
\end{figure}

\subsection{RBM Learning}
\paragraph{Experiment Design.}
We use the same RBM structure as the sampling task, with 500 hidden units and 784 visible units. We apply the samplers of interest to the PCD algorithm introduced by~\citet{tieleman2008training}. The model parameters are tuned via the Adam optimizer with a learning rate of
.001.

\paragraph{Sampler Configuration.}
For DMALA, we set step size to $0.2$ , and for AB we use the default hyperparameters for the first order sampler. For ACS, we follow the setting in \citet{pynadath2024gradient} and use $\rho^*=0.5, \beta_{\max }=0.95, \zeta=0.5$, cycle length $s=20$ for all the datasets. We also fix the total overhead of the tuning algorithm to $10 \%$ of the total sampling steps. For PT-DMALA, based on the results of the pilot run, we set the number of chains between 3 and 5, and assigned temperatures accordingly, with step sizes ranging from $0.15$ to $0.4$ across chains.


\subsection{Learning EBMs}
\paragraph{Experiment Setup.} We adopt the same ResNet structure and experiment protocol as in~\citet{grathwohl2021oops}, where the network has 8 residual blocks with 64 feature maps. There are 2 convolutional layers for each residual block. The network uses Swish activation function~\citep{ramachandran2017searching}. For static/dynamic MNIST and Omniglot, we use a replay buffer with 10,000 samples. For Caltech, we use a replay buffer with 1,000 samples. We evaluate the models every 5,000 iterations by running AIS for 10,000 steps. The reported results are from the model which performs the best on the validation set. The final reported numbers are generated by running 300,000 iterations of AIS. All the models are trained with Adam~\citep{kingma2014adam} with a learning rate of 0.001 for 50,000 iterations.

\paragraph{Sampler Configuration.} For DMALA, we use a step size of $0.15$ as used in~\citet{zhang2022langevin}. For ACS, we follow the setting in~\citet{pynadath2024gradient} and use $200$ sampling steps for EstimateAlphaMax and EstimateAlphaMin. For Static MNIST, Dynamic MNIST, and Omniglot, we set the algorithm to tune $\alpha_{\max }$ and $\alpha_{\min }$ every $25$ cycles, where each cycle has $50$ training iterations. For Caltech Silhouettes, we have to adapt every $10$ cycles with the same number of training iterations. We set the step sizes as $0.05 \sim 0.4$ for all chains in our algorithm. 

\paragraph{Generated Images.}
Here we provide the generated results in~\cref{img_ebm_learn_genimgs} from our algorithm across Static MNIST, Dynamic MNIST, Omniglot, and Caltech Silhouettes. These images demonstrate the ability of trained deep EBMs to capture the underlying data distribution. The deep EBM is capable of producing high-quality samples that visually resemble the training data, which indicates that the learned energy function effectively models the complex, high-dimensional structure of the data.
\begin{figure}[htbp]
\def\imgvspace{0.25cm}
\def\imghspace{.25cm}
\def\imghwidth{3.15cm}
\centering
\subfloat{\includegraphics[width=\imghwidth]{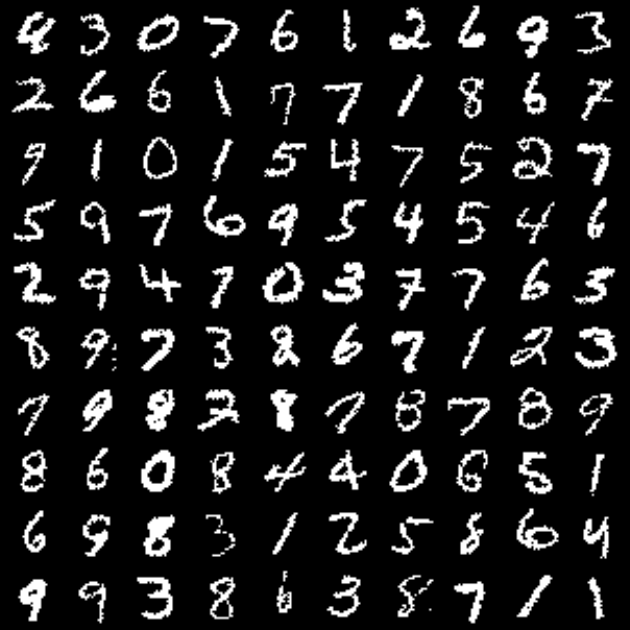}}\hspace*{\imghspace}
\subfloat{\includegraphics[width=\imghwidth]{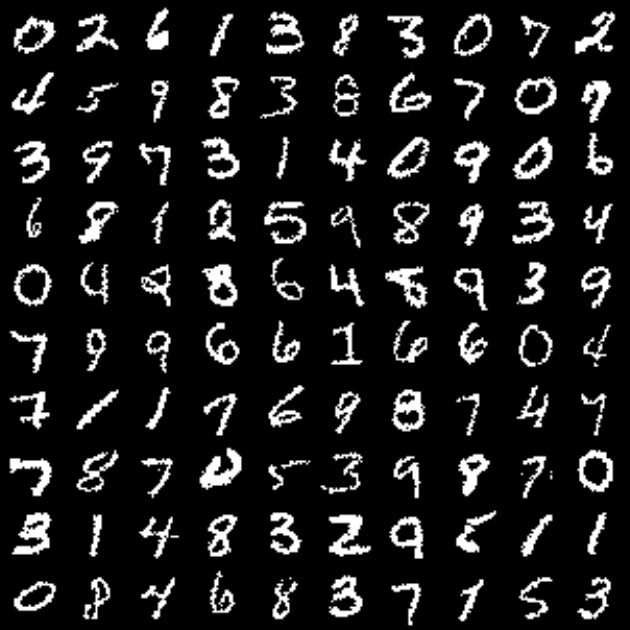}}\hspace*{\imghspace}
\subfloat{\includegraphics[width=\imghwidth]{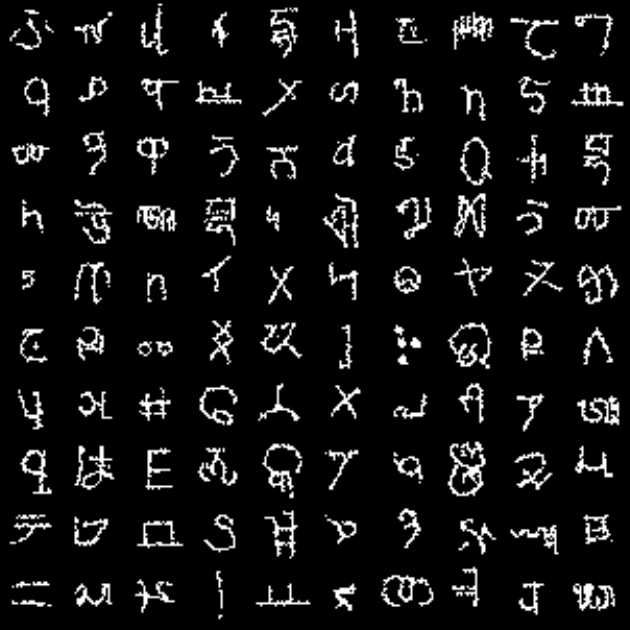}}\hspace*{\imghspace}
\subfloat{\includegraphics[width=\imghwidth]{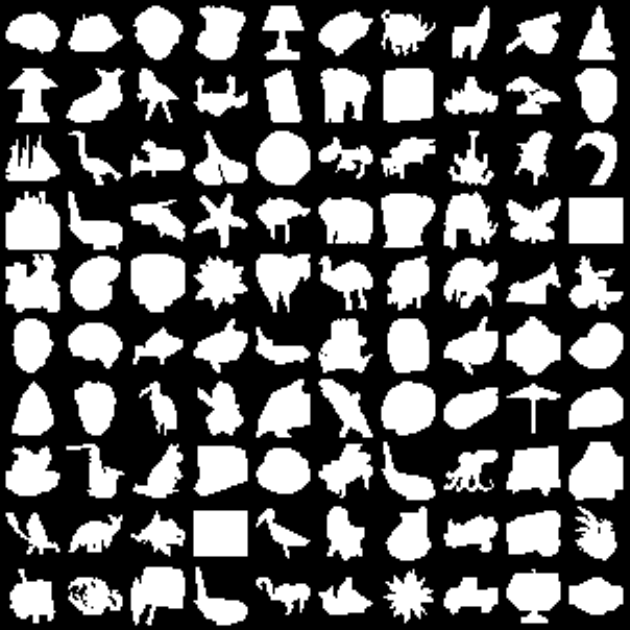}} 

\vspace{\imgvspace}

\centering
\hspace{.15cm}
\subfloat[Static]{\includegraphics[width=\imghwidth]{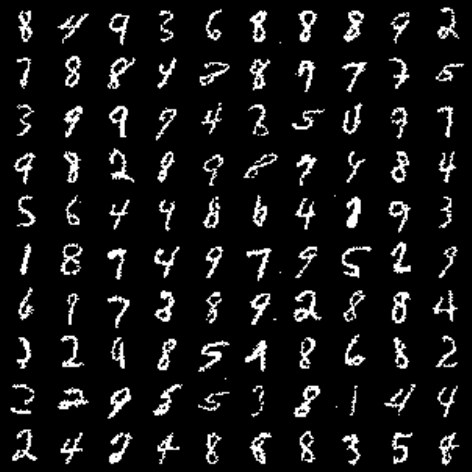}}\hspace*{\imghspace}
\subfloat[Dynamic]{\includegraphics[width=\imghwidth]{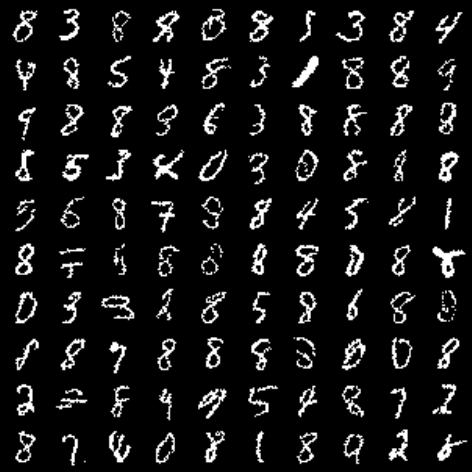}}\hspace*{\imghspace}
\subfloat[Omniglot]{\includegraphics[width=\imghwidth]{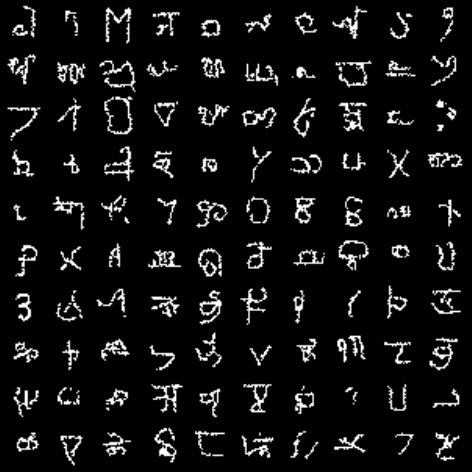}}\hspace*{\imghspace}
\subfloat[Caltech]{\includegraphics[width=\imghwidth]{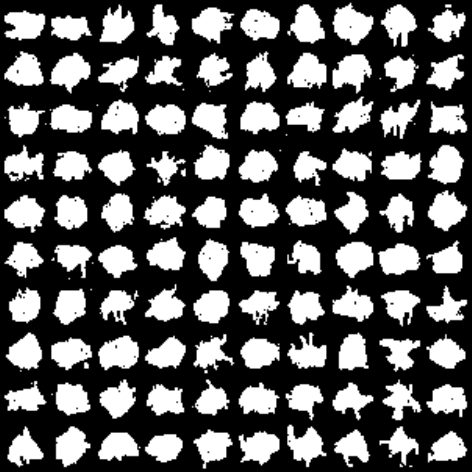}} \hspace*{\imghspace}
\caption{The images on the top row are examples from the dataset, while the bottom row are from the trained EBM. The images generated from our algorithm are similar to those from the dataset, demonstrating that the model is capable of generating high-quality samples.}
\label{img_ebm_learn_genimgs}
\end{figure}


\newpage
\section*{NeurIPS Paper Checklist}
\begin{enumerate}

\item {\bf Claims}
    \item[] Question: Do the main claims made in the abstract and introduction accurately reflect the paper's contributions and scope?
    \item[] Answer: \answerYes{} 
    \item[] Justification: The main claims made in the abstract and introduction accurately reflect the paper’s contributions and scope. See the Abstract and Introduction sections.
    \item[] Guidelines:
    \begin{itemize}
        \item The answer NA means that the abstract and introduction do not include the claims made in the paper.
        \item The abstract and/or introduction should clearly state the claims made, including the contributions made in the paper and important assumptions and limitations. A No or NA answer to this question will not be perceived well by the reviewers. 
        \item The claims made should match theoretical and experimental results, and reflect how much the results can be expected to generalize to other settings. 
        \item It is fine to include aspirational goals as motivation as long as it is clear that these goals are not attained by the paper. 
    \end{itemize}

\item {\bf Limitations}
    \item[] Question: Does the paper discuss the limitations of the work performed by the authors?
    \item[] Answer: \answerYes{} 
    \item[] Justification: The limitation of the work is discussed in the Limitations section.
    \item[] Guidelines:
    \begin{itemize}
        \item The answer NA means that the paper has no limitation while the answer No means that the paper has limitations, but those are not discussed in the paper. 
        \item The authors are encouraged to create a separate "Limitations" section in their paper.
        \item The paper should point out any strong assumptions and how robust the results are to violations of these assumptions (e.g., independence assumptions, noiseless settings, model well-specification, asymptotic approximations only holding locally). The authors should reflect on how these assumptions might be violated in practice and what the implications would be.
        \item The authors should reflect on the scope of the claims made, e.g., if the approach was only tested on a few datasets or with a few runs. In general, empirical results often depend on implicit assumptions, which should be articulated.
        \item The authors should reflect on the factors that influence the performance of the approach. For example, a facial recognition algorithm may perform poorly when image resolution is low or images are taken in low lighting. Or a speech-to-text system might not be used reliably to provide closed captions for online lectures because it fails to handle technical jargon.
        \item The authors should discuss the computational efficiency of the proposed algorithms and how they scale with dataset size.
        \item If applicable, the authors should discuss possible limitations of their approach to address problems of privacy and fairness.
        \item While the authors might fear that complete honesty about limitations might be used by reviewers as grounds for rejection, a worse outcome might be that reviewers discover limitations that aren't acknowledged in the paper. The authors should use their best judgment and recognize that individual actions in favor of transparency play an important role in developing norms that preserve the integrity of the community. Reviewers will be specifically instructed to not penalize honesty concerning limitations.
    \end{itemize}

\item {\bf Theory assumptions and proofs}
    \item[] Question: For each theoretical result, does the paper provide the full set of assumptions and a complete (and correct) proof?
    \item[] Answer: \answerYes{} 
    \item[] Justification: For all theoretical results, the paper provides the corresponding proofs in the appendix.
    \item[] Guidelines:
    \begin{itemize}
        \item The answer NA means that the paper does not include theoretical results. 
        \item All the theorems, formulas, and proofs in the paper should be numbered and cross-referenced.
        \item All assumptions should be clearly stated or referenced in the statement of any theorems.
        \item The proofs can either appear in the main paper or the supplemental material, but if they appear in the supplemental material, the authors are encouraged to provide a short proof sketch to provide intuition. 
        \item Inversely, any informal proof provided in the core of the paper should be complemented by formal proofs provided in appendix or supplemental material.
        \item Theorems and Lemmas that the proof relies upon should be properly referenced. 
    \end{itemize}

    \item {\bf Experimental result reproducibility}
    \item[] Question: Does the paper fully disclose all the information needed to reproduce the main experimental results of the paper to the extent that it affects the main claims and/or conclusions of the paper (regardless of whether the code and data are provided or not)?
    \item[] Answer: \answerYes{} 
    \item[] Justification: The paper fully discloses all the information needed to reproduce the main experimental results in the paper. See the Experiments section.
    \item[] Guidelines:
    \begin{itemize}
        \item The answer NA means that the paper does not include experiments.
        \item If the paper includes experiments, a No answer to this question will not be perceived well by the reviewers: Making the paper reproducible is important, regardless of whether the code and data are provided or not.
        \item If the contribution is a dataset and/or model, the authors should describe the steps taken to make their results reproducible or verifiable. 
        \item Depending on the contribution, reproducibility can be accomplished in various ways. For example, if the contribution is a novel architecture, describing the architecture fully might suffice, or if the contribution is a specific model and empirical evaluation, it may be necessary to either make it possible for others to replicate the model with the same dataset, or provide access to the model. In general. releasing code and data is often one good way to accomplish this, but reproducibility can also be provided via detailed instructions for how to replicate the results, access to a hosted model (e.g., in the case of a large language model), releasing of a model checkpoint, or other means that are appropriate to the research performed.
        \item While NeurIPS does not require releasing code, the conference does require all submissions to provide some reasonable avenue for reproducibility, which may depend on the nature of the contribution. For example
        \begin{enumerate}
            \item If the contribution is primarily a new algorithm, the paper should make it clear how to reproduce that algorithm.
            \item If the contribution is primarily a new model architecture, the paper should describe the architecture clearly and fully.
            \item If the contribution is a new model (e.g., a large language model), then there should either be a way to access this model for reproducing the results or a way to reproduce the model (e.g., with an open-source dataset or instructions for how to construct the dataset).
            \item We recognize that reproducibility may be tricky in some cases, in which case authors are welcome to describe the particular way they provide for reproducibility. In the case of closed-source models, it may be that access to the model is limited in some way (e.g., to registered users), but it should be possible for other researchers to have some path to reproducing or verifying the results.
        \end{enumerate}
    \end{itemize}

\item {\bf Open access to data and code}
    \item[] Question: Does the paper provide open access to the data and code, with sufficient instructions to faithfully reproduce the main experimental results, as described in supplemental material?
    \item[] Answer: \answerYes{} 
    \item[] Justification: We provide the data and code in the supplemental material to reproduce the main experimental results.
    \item[] Guidelines:
    \begin{itemize}
        \item The answer NA means that paper does not include experiments requiring code.
        \item Please see the NeurIPS code and data submission guidelines (\url{https://nips.cc/public/guides/CodeSubmissionPolicy}) for more details.
        \item While we encourage the release of code and data, we understand that this might not be possible, so “No” is an acceptable answer. Papers cannot be rejected simply for not including code, unless this is central to the contribution (e.g., for a new open-source benchmark).
        \item The instructions should contain the exact command and environment needed to run to reproduce the results. See the NeurIPS code and data submission guidelines (\url{https://nips.cc/public/guides/CodeSubmissionPolicy}) for more details.
        \item The authors should provide instructions on data access and preparation, including how to access the raw data, preprocessed data, intermediate data, and generated data, etc.
        \item The authors should provide scripts to reproduce all experimental results for the new proposed method and baselines. If only a subset of experiments are reproducible, they should state which ones are omitted from the script and why.
        \item At submission time, to preserve anonymity, the authors should release anonymized versions (if applicable).
        \item Providing as much information as possible in supplemental material (appended to the paper) is recommended, but including URLs to data and code is permitted.
    \end{itemize}

\item {\bf Experimental setting/details}
    \item[] Question: Does the paper specify all the training and test details (e.g., data splits, hyperparameters, how they were chosen, type of optimizer, etc.) necessary to understand the results?
    \item[] Answer: \answerYes{} 
    \item[] Justification: We specify all the training and test details necessary to understand the results. See the Experiments section.
    \item[] Guidelines:
    \begin{itemize}
        \item The answer NA means that the paper does not include experiments.
        \item The experimental setting should be presented in the core of the paper to a level of detail that is necessary to appreciate the results and make sense of them.
        \item The full details can be provided either with the code, in appendix, or as supplemental material.
    \end{itemize}

\item {\bf Experiment statistical significance}
    \item[] Question: Does the paper report error bars suitably and correctly defined or other appropriate information about the statistical significance of the experiments?
    \item[] Answer: \answerYes{} 
    \item[] Justification: We report the statistical significance of the experiments. In the tables presenting the experimental results, we provide the standard deviations across multiple runs.
    \item[] Guidelines:
    \begin{itemize}
        \item The answer NA means that the paper does not include experiments.
        \item The authors should answer "Yes" if the results are accompanied by error bars, confidence intervals, or statistical significance tests, at least for the experiments that support the main claims of the paper.
        \item The factors of variability that the error bars are capturing should be clearly stated (for example, train/test split, initialization, random drawing of some parameter, or overall run with given experimental conditions).
        \item The method for calculating the error bars should be explained (closed form formula, call to a library function, bootstrap, etc.)
        \item The assumptions made should be given (e.g., Normally distributed errors).
        \item It should be clear whether the error bar is the standard deviation or the standard error of the mean.
        \item It is OK to report 1-sigma error bars, but one should state it. The authors should preferably report a 2-sigma error bar than state that they have a 96\% CI, if the hypothesis of Normality of errors is not verified.
        \item For asymmetric distributions, the authors should be careful not to show in tables or figures symmetric error bars that would yield results that are out of range (e.g. negative error rates).
        \item If error bars are reported in tables or plots, The authors should explain in the text how they were calculated and reference the corresponding figures or tables in the text.
    \end{itemize}

\item {\bf Experiments compute resources}
    \item[] Question: For each experiment, does the paper provide sufficient information on the computer resources (type of compute workers, memory, time of execution) needed to reproduce the experiments?
    \item[] Answer: \answerYes{} 
    \item[] Justification: All experiments are conducted on a normal laptop and we provide the time of execution needed to reproduce the experiments. See the Experiments section.
    \item[] Guidelines:
    \begin{itemize}
        \item The answer NA means that the paper does not include experiments.
        \item The paper should indicate the type of compute workers CPU or GPU, internal cluster, or cloud provider, including relevant memory and storage.
        \item The paper should provide the amount of compute required for each of the individual experimental runs as well as estimate the total compute. 
        \item The paper should disclose whether the full research project required more compute than the experiments reported in the paper (e.g., preliminary or failed experiments that didn't make it into the paper). 
    \end{itemize}
    
\item {\bf Code of ethics}
    \item[] Question: Does the research conducted in the paper conform, in every respect, with the NeurIPS Code of Ethics \url{https://neurips.cc/public/EthicsGuidelines}?
    \item[] Answer: \answerYes{} 
    \item[] Justification: The research conducted in the paper conform with the NeurIPS Code of Ethics.
    \item[] Guidelines:
    \begin{itemize}
        \item The answer NA means that the authors have not reviewed the NeurIPS Code of Ethics.
        \item If the authors answer No, they should explain the special circumstances that require a deviation from the Code of Ethics.
        \item The authors should make sure to preserve anonymity (e.g., if there is a special consideration due to laws or regulations in their jurisdiction).
    \end{itemize}

\item {\bf Broader impacts}
    \item[] Question: Does the paper discuss both potential positive societal impacts and negative societal impacts of the work performed?
    \item[] Answer: \answerNo{} 
    \item[] Justification:   This paper presents work whose goal is to advance the field of machine learning. There are many potential societal consequences of our work, none of which we feel must be specifically highlighted here.
    \item[] Guidelines:
    \begin{itemize}
        \item The answer NA means that there is no societal impact of the work performed.
        \item If the authors answer NA or No, they should explain why their work has no societal impact or why the paper does not address societal impact.
        \item Examples of negative societal impacts include potential malicious or unintended uses (e.g., disinformation, generating fake profiles, surveillance), fairness considerations (e.g., deployment of technologies that could make decisions that unfairly impact specific groups), privacy considerations, and security considerations.
        \item The conference expects that many papers will be foundational research and not tied to particular applications, let alone deployments. However, if there is a direct path to any negative applications, the authors should point it out. For example, it is legitimate to point out that an improvement in the quality of generative models could be used to generate deepfakes for disinformation. On the other hand, it is not needed to point out that a generic algorithm for optimizing neural networks could enable people to train models that generate Deepfakes faster.
        \item The authors should consider possible harms that could arise when the technology is being used as intended and functioning correctly, harms that could arise when the technology is being used as intended but gives incorrect results, and harms following from (intentional or unintentional) misuse of the technology.
        \item If there are negative societal impacts, the authors could also discuss possible mitigation strategies (e.g., gated release of models, providing defenses in addition to attacks, mechanisms for monitoring misuse, mechanisms to monitor how a system learns from feedback over time, improving the efficiency and accessibility of ML).
    \end{itemize}
    
\item {\bf Safeguards}
    \item[] Question: Does the paper describe safeguards that have been put in place for responsible release of data or models that have a high risk for misuse (e.g., pretrained language models, image generators, or scraped datasets)?
    \item[] Answer: \answerNA{} 
    \item[] Justification: The paper poses no such risks.
    \item[] Guidelines:
    \begin{itemize}
        \item The answer NA means that the paper poses no such risks.
        \item Released models that have a high risk for misuse or dual-use should be released with necessary safeguards to allow for controlled use of the model, for example by requiring that users adhere to usage guidelines or restrictions to access the model or implementing safety filters. 
        \item Datasets that have been scraped from the Internet could pose safety risks. The authors should describe how they avoided releasing unsafe images.
        \item We recognize that providing effective safeguards is challenging, and many papers do not require this, but we encourage authors to take this into account and make a best faith effort.
    \end{itemize}

\item {\bf Licenses for existing assets}
    \item[] Question: Are the creators or original owners of assets (e.g., code, data, models), used in the paper, properly credited and are the license and terms of use explicitly mentioned and properly respected?
    \item[] Answer: \answerYes{} 
    \item[] Justification: All code, models, and datasets mentioned in the text are appropriately cited with their original papers.
    \item[] Guidelines:
    \begin{itemize}
        \item The answer NA means that the paper does not use existing assets.
        \item The authors should cite the original paper that produced the code package or dataset.
        \item The authors should state which version of the asset is used and, if possible, include a URL.
        \item The name of the license (e.g., CC-BY 4.0) should be included for each asset.
        \item For scraped data from a particular source (e.g., website), the copyright and terms of service of that source should be provided.
        \item If assets are released, the license, copyright information, and terms of use in the package should be provided. For popular datasets, \url{paperswithcode.com/datasets} has curated licenses for some datasets. Their licensing guide can help determine the license of a dataset.
        \item For existing datasets that are re-packaged, both the original license and the license of the derived asset (if it has changed) should be provided.
        \item If this information is not available online, the authors are encouraged to reach out to the asset's creators.
    \end{itemize}

\item {\bf New assets}
    \item[] Question: Are new assets introduced in the paper well documented and is the documentation provided alongside the assets?
    \item[] Answer: \answerYes{} 
    \item[] Justification: New assets introduced in the paper, such as code, are well documented. The documentation is provided alongside the assets in the supplementary material.
    \item[] Guidelines:
    \begin{itemize}
        \item The answer NA means that the paper does not release new assets.
        \item Researchers should communicate the details of the dataset/code/model as part of their submissions via structured templates. This includes details about training, license, limitations, etc. 
        \item The paper should discuss whether and how consent was obtained from people whose asset is used.
        \item At submission time, remember to anonymize your assets (if applicable). You can either create an anonymized URL or include an anonymized zip file.
    \end{itemize}

\item {\bf Crowdsourcing and research with human subjects}
    \item[] Question: For crowdsourcing experiments and research with human subjects, does the paper include the full text of instructions given to participants and screenshots, if applicable, as well as details about compensation (if any)? 
    \item[] Answer: \answerNA{} 
    \item[] Justification: The paper does not involve crowdsourcing nor research with human subjects.
    \item[] Guidelines:
    \begin{itemize}
        \item The answer NA means that the paper does not involve crowdsourcing nor research with human subjects.
        \item Including this information in the supplemental material is fine, but if the main contribution of the paper involves human subjects, then as much detail as possible should be included in the main paper. 
        \item According to the NeurIPS Code of Ethics, workers involved in data collection, curation, or other labor should be paid at least the minimum wage in the country of the data collector. 
    \end{itemize}

\item {\bf Institutional review board (IRB) approvals or equivalent for research with human subjects}
    \item[] Question: Does the paper describe potential risks incurred by study participants, whether such risks were disclosed to the subjects, and whether Institutional Review Board (IRB) approvals (or an equivalent approval/review based on the requirements of your country or institution) were obtained?
    \item[] Answer: \answerNA{} 
    \item[] Justification: The paper does not involve crowdsourcing nor research with human subjects.
    \item[] Guidelines:
    \begin{itemize}
        \item The answer NA means that the paper does not involve crowdsourcing nor research with human subjects.
        \item Depending on the country in which research is conducted, IRB approval (or equivalent) may be required for any human subjects research. If you obtained IRB approval, you should clearly state this in the paper. 
        \item We recognize that the procedures for this may vary significantly between institutions and locations, and we expect authors to adhere to the NeurIPS Code of Ethics and the guidelines for their institution. 
        \item For initial submissions, do not include any information that would break anonymity (if applicable), such as the institution conducting the review.
    \end{itemize}

\item {\bf Declaration of LLM usage}
    \item[] Question: Does the paper describe the usage of LLMs if it is an important, original, or non-standard component of the core methods in this research? Note that if the LLM is used only for writing, editing, or formatting purposes and does not impact the core methodology, scientific rigorousness, or originality of the research, declaration is not required.
    \item[] Answer: \answerNA{} 
    \item[] Justification: The development of our algorithm does not involve LLMs.
    \item[] Guidelines:
    \begin{itemize}
        \item The answer NA means that the core method development in this research does not involve LLMs as any important, original, or non-standard components.
        \item Please refer to our LLM policy (\url{https://neurips.cc/Conferences/2025/LLM}) for what should or should not be described.
    \end{itemize}

\end{enumerate}

\end{document}